\documentclass[10pt]{article} 

\usepackage[table]{xcolor}
\usepackage[utf8]{inputenc}
\usepackage{amsmath,amssymb,amsthm}
\usepackage{a4wide}
\usepackage{hyperref}
\usepackage{natbib}
\usepackage{authblk}
\usepackage{graphicx}
\usepackage{enumitem}
\setlist[itemize]{leftmargin=*, nosep}
\setlist[enumerate]{leftmargin=*, nosep}

\usepackage{tikz}
\usetikzlibrary{arrows.meta, positioning, shapes.geometric, shapes.multipart,
  fit, backgrounds, calc, decorations.pathreplacing, decorations.markings,
  patterns, shadows}

\usepackage[most]{tcolorbox}

\newtheorem{theorem_inner}{Theorem}[section]
\newtheorem{proposition_inner}[theorem_inner]{Proposition}
\newtheorem{definition_inner}[theorem_inner]{Definition}

\newenvironment{theorem}
  {\begin{tcolorbox}[colback=gray!8, colframe=gray!40, rounded corners=all,
    arc=3pt, boxrule=0.4pt, top=4pt, bottom=4pt, left=6pt, right=6pt]
   \begin{theorem_inner}}
  {\end{theorem_inner}\end{tcolorbox}}

\newenvironment{proposition}
  {\begin{tcolorbox}[colback=gray!8, colframe=gray!40, rounded corners=all,
    arc=3pt, boxrule=0.4pt, top=4pt, bottom=4pt, left=6pt, right=6pt]
   \begin{proposition_inner}}
  {\end{proposition_inner}\end{tcolorbox}}

\newenvironment{definition}
  {\begin{tcolorbox}[colback=gray!8, colframe=gray!40, rounded corners=all,
    arc=3pt, boxrule=0.4pt, top=4pt, bottom=4pt, left=6pt, right=6pt]
   \begin{definition_inner}}
  {\end{definition_inner}\end{tcolorbox}}

\title{\LARGE \textbf{A Survey of Safe Reinforcement Learning and Constrained MDPs: \ A Technical Survey on Single-Agent and Multi-Agent Safety}}

\author[1]{Ankita Kushwaha}
\author[1]{Kiran Ravish}
\author[1]{Preeti Lamba}
\author[1]{Pawan Kumar}
\author[2]{Anuj Mahajan}
\affil[1]{International Institute of Information Technology, Hyderabad, India}
\affil[2]{Meta SuperIntelligence, USA}

\begin{document}

\maketitle
\begin{abstract}
\begin{tcolorbox}[colback=gray!8, colframe=gray!30, rounded corners=all,
  arc=4pt, boxrule=0.3pt, top=6pt, bottom=6pt, left=8pt, right=8pt]
Safe Reinforcement Learning (SafeRL) is the subfield of reinforcement learning that explicitly deals with safety constraints during the learning and deployment of agents. This survey provides a mathematically rigorous overview of SafeRL formulations based on Constrained Markov Decision Processes (CMDPs) and extensions to Multi-Agent Safe RL (SafeMARL). We review theoretical foundations of CMDPs, covering definitions, constrained optimization techniques, and fundamental theorems. We then summarize state-of-the-art algorithms in SafeRL for single agents, including policy gradient methods with safety guarantees and safe exploration strategies, as well as recent advances in SafeMARL for cooperative and competitive settings. Additionally, we propose five open research problems to advance the field, with three focusing on SafeMARL. Each problem is described with motivation, key challenges, and related prior work. This survey is intended as a technical guide for researchers interested in SafeRL and SafeMARL, highlighting key concepts, methods, and open future research directions.
\end{tcolorbox}
\end{abstract}

\tableofcontents

\section{Introduction} Reinforcement learning (RL) has achieved remarkable success in domains such as games, robotics, and autonomous systems. However, when deploying RL in real-world \emph{safety-critical} applications (e.g., autonomous driving, healthcare, robotics), it is essential to ensure that the learning agent avoids catastrophic failures or unsafe behaviors \cite{Amodei2016,Garcia2015}.  \textbf{Safe Reinforcement Learning (SafeRL)} addresses this need by augmenting standard RL objectives with safety considerations, typically in the form of constraints on the agent’s behavior or environment outcomes. 

\begin{definition}
The goal in SafeRL is to maximize performance (cumulative reward) while satisfying safety constraints during training and deployment.
\end{definition}

A common framework for SafeRL is the \textbf{Constrained Markov Decision Process (CMDP)} introduced by \cite{Altman1999}. In a CMDP, an agent seeks to maximize expected return subject to one or more constraints (e.g., bounds on certain costs or probabilities of unsafe events). This framework allows formalizing safety requirements as mathematical constraints and provides tools from constrained optimization and control theory to enforce them. SafeRL algorithms often leverage CMDP theory to find policies that respect constraints (at least approximately) while learning efficiently. SafeRL has gained significant attention in recent years. Early work in SafeRL explored modifications of the RL objective to encode risk or safety (e.g., worst-case guarantees \cite{Heger1994}, risk-sensitive criteria \cite{Borkar2002}, or probability of failure constraints \cite{Geibel2005}). More recent approaches explicitly enforce constraints during learning using techniques like Lagrange multipliers, trust-region methods, or safety monitors. There have been comprehensive surveys of SafeRL (e.g., \cite{Garcia2015}) and increasing theoretical study of constrained RL algorithms \cite{Achiam2017,Chow2018}. An emerging frontier is \textbf{Multi-Agent Safe Reinforcement Learning (SafeMARL)}, which considers multiple agents learning and interacting under safety constraints. SafeMARL is crucial for applications like coordinated robotics, drone swarms, and autonomous driving with multiple vehicles, where safety conditions involve interactions among agents. SafeMARL introduces additional challenges such as coordinating safety in a team, handling the coupling of constraints across agents, and new solution concepts (like safe equilibria \cite{Ganzfried2022} in competitive settings). While single-agent SafeRL is relatively well-studied, SafeMARL remains a young research area with many open problems \cite{ElSayed2021,Gu2023MultiRobot}. This survey provides 

\begin{itemize} 
\item A rigorous introduction to SafeRL formulations based on CMDPs, including mathematical definitions and theorems. 
\item A review of state-of-the-art SafeRL methods for single agents, and their extensions to multi-agent scenarios (SafeMARL), highlighting major algorithms and theoretical guarantees. 
\item A discussion of related work and different perspectives on safety in RL (e.g., risk-sensitive RL, robust RL, safe exploration techniques). 
\item Five open research problems that we believe are important for advancing SafeRL and SafeMARL. Three of these focus specifically on challenges in SafeMARL. \end{itemize}

Our target audience is researchers familiar with fundamental RL concepts who seek a deeper understanding of how to incorporate safety in RL. We assume knowledge of basic RL (Markov decision processes, policy optimization, etc.) and provide definitions and notation for SafeRL topics. We believe that by the end of this paper, the reader should be equipped with the theoretical background of CMDPs, knowledge of leading algorithms in SafeRL/SafeMARL, and insight into promising research directions in this field.

\begin{figure}[t!]
\centering
\begin{tikzpicture}[
  every node/.style={font=\footnotesize},
  mainbox/.style={rectangle, draw=blue!70!black, fill=blue!15, rounded corners=3pt,
    minimum height=0.7cm, minimum width=2.2cm, align=center, font=\footnotesize\bfseries,
    line width=0.8pt},
  secbox/.style={rectangle, draw=teal!70!black, fill=teal!8, rounded corners=2pt,
    minimum height=0.55cm, minimum width=1.6cm, align=center, font=\scriptsize,
    line width=0.6pt},
  leafbox/.style={rectangle, draw=gray!60, fill=gray!5, rounded corners=2pt,
    minimum height=0.45cm, align=center, font=\scriptsize, line width=0.5pt},
  arr/.style={-{Stealth[length=2mm]}, line width=0.6pt, gray!70!black},
]

\node[mainbox, minimum width=4cm] (top) at (0,0) {A Survey of SafeRL and CMDPs};

\node[mainbox] (sec3) at (-4.8,-1.3) {Foundations\\[-1pt](Sec.~3)};
\node[mainbox] (sec4) at (0,-1.3) {Methods\\[-1pt](Sec.~4)};
\node[mainbox] (sec5) at (4.8,-1.3) {Open Problems\\[-1pt](Sec.~5)};

\draw[arr] (top.south) -- ++(0,-0.2) -| (sec3.north);
\draw[arr] (top.south) -- (sec4.north);
\draw[arr] (top.south) -- ++(0,-0.2) -| (sec5.north);

\node[secbox, minimum width=2.2cm] (mdp)   at (-4.8,-2.6) {MDPs};
\node[secbox, minimum width=2.2cm] (cmdp)  at (-4.8,-3.3) {CMDPs};
\node[secbox, minimum width=2.2cm] (ctype) at (-4.8,-4.0) {Constraint Types};
\node[leafbox, minimum width=2.2cm] (lagr)  at (-4.8,-4.7) {Lagrangian \& LP};
\node[leafbox, minimum width=2.2cm] (thm)   at (-4.8,-5.3) {Theorems};

\draw[arr] (sec3.south) -- (mdp.north);
\draw[arr, gray!50] (mdp.south) -- (cmdp.north);
\draw[arr, gray!50] (cmdp.south) -- (ctype.north);
\draw[arr, gray!50] (ctype.south) -- (lagr.north);
\draw[arr, gray!50] (lagr.south) -- (thm.north);

\node[secbox, minimum width=2.2cm] (lag4)    at (0,-2.6) {Lagrangian Methods};
\node[secbox, minimum width=2.2cm] (cpo4)    at (0,-3.3) {CPO / Trust Region};
\node[secbox, minimum width=2.2cm] (shield4) at (0,-4.0) {Safety Shields};
\node[secbox, minimum width=2.2cm] (marl4)   at (0,-4.7) {SafeMARL};

\draw[arr] (sec4.south) -- (lag4.north);
\draw[arr, gray!50] (lag4.south) -- (cpo4.north);
\draw[arr, gray!50] (cpo4.south) -- (shield4.north);
\draw[arr, gray!50] (shield4.south) -- (marl4.north);

\node[leafbox, minimum width=2.2cm] (cent)   at (0,-5.4) {Centralized (MACPO)};
\node[leafbox, minimum width=2.2cm] (decent) at (0,-6.0) {Decentralized};
\draw[arr, gray!50] (marl4.south) -- (cent.north);
\draw[arr, gray!50] (cent.south) -- (decent.north);

\node[leafbox, minimum width=2.6cm] (p1) at (4.8,-2.6) {P1: Zero-violation};
\node[leafbox, minimum width=2.6cm] (p2) at (4.8,-3.2) {P2: Partial Observability};
\node[leafbox, minimum width=2.6cm] (p3) at (4.8,-3.8) {P3: Decentral.\ SafeMARL};
\node[leafbox, minimum width=2.6cm] (p4) at (4.8,-4.4) {P4: Competitive SafeMARL};
\node[leafbox, minimum width=2.6cm] (p5) at (4.8,-5.0) {P5: Non-stationarity};

\draw[arr] (sec5.south) -- (p1.north);
\draw[arr, gray!50] (p1.south) -- (p2.north);
\draw[arr, gray!50] (p2.south) -- (p3.north);
\draw[arr, gray!50] (p3.south) -- (p4.north);
\draw[arr, gray!50] (p4.south) -- (p5.north);

\end{tikzpicture}
\caption{Overview and structure of this survey. The paper covers theoretical foundations of SafeRL via CMDPs (Sec.~3), state-of-the-art methods for single-agent and multi-agent settings (Sec.~4), and five open research problems (Sec.~5).}
\label{fig:survey_roadmap}
\end{figure}

\section{Related Work} SafeRL has been surveyed and reviewed from multiple angles. Garc{'\i}a and Fern{'a}ndez \cite{Garcia2015} provide an earlier comprehensive survey of SafeRL methods up to 2015, categorizing approaches into modifications of the optimality criterion (e.g., constrained or risk-sensitive objectives) and modifications of the exploration process (e.g., using external knowledge or risk metrics to guide learning). They classify safety criteria into four groups: 

\begin{itemize} 

\item \emph{Constrained criteria} – optimization with explicit constraints on policies \cite{Geibel2005}, 

\item \emph{Worst-case (robust) criteria} – optimize the minimal possible return under adversarial conditions \cite{Heger1994}, 

\item \emph{Risk-sensitive criteria} – incorporate risk measures like variance or CVaR (Conditional Value-at-Risk) into the objective \cite{Borkar2002,Tamar2015}, 

\item \emph{Others} – e.g., criteria based on higher moments or probability of ruin. 

\end{itemize} 

Our survey focuses primarily on the constrained criterion approach (CMDPs), which has become the prevalent formalism for SafeRL in recent years. Since 2015, the field has advanced with new algorithms and theoretical results. Recent reviews such as Wachi \emph{et al.} \cite{Wachi2024} examine various formulations of safety constraints (e.g., how constraints are represented and enforced) and draw connections between them. Another forthcoming survey by Gu \emph{et al.} \cite{Gu2024Review} provides an extensive review of SafeRL methods, theory, and applications, reflecting the growing maturity of the field. Domain-specific surveys have also emerged, such as Brunke \emph{et al.} \cite{brunke2021srlrobotics} who provide a comprehensive review of safe RL methods for robotics. These works indicate an increasing interest in unifying SafeRL concepts and developing a systematic understanding of safety constraint representations and their implications. On the multi-agent side, SafeMARL has been less surveyed due to its emergent status. Gu \emph{et al.} \cite{Gu2023MultiRobot} investigate safe multi-robot control tasks and propose algorithms like Multi-Agent Constrained Policy Optimization (MACPO). Some recent papers introduce safe multi-agent learning algorithms or frameworks \cite{ElSayed2021,Gu2023MultiRobot,Zhang2024Scalable}, but a comprehensive survey of SafeMARL is still lacking. Our work contributes by reviewing both single-agent and multi-agent safe RL in one document and highlighting SafeMARL-specific challenges. Other related areas include \textbf{robust RL} (handling model uncertainty or adversarial disturbances) and \textbf{reward hacking / alignment} (ensuring the specified reward leads to intended safe behavior). While robust RL (e.g., solving worst-case MDPs) and SafeRL share some techniques (like min-max optimization), they address different problem formulations (uncertainty vs. explicit constraints). Similarly, reward specification and alignment problems are complementary to SafeRL: one can combine learned reward shaping with SafeRL constraints to yield agents that both seek correct objectives and stay safe \cite{Amodei2016,Achiam2017}. Benchmark suites such as the AI Safety Gridworlds \cite{leike2017aisafetygridworlds} and SafeLife \cite{wainwright2021safelife10exploringeffects} specifically test for specification robustness and side-effect avoidance. We touch upon these connections where relevant. In summary, our survey builds upon and extends prior work by providing a focused treatment of CMDP-based SafeRL and the novel area of SafeMARL, presented in a rigorous yet accessible manner for researchers.

\section{Safe Reinforcement Learning and Constrained MDPs: Foundations} In this section, we introduce the theoretical foundations of SafeRL with an emphasis on Constrained Markov Decision Processes (CMDPs). We present formal definitions, notation, and key mathematical results that underpin SafeRL algorithms. We also discuss how safety constraints are formulated and how they can be tackled using constrained optimization techniques in an RL context.

\subsection{Markov Decision Processes (MDPs)}

\begin{definition}[Markov Decision Process]
We begin with the standard Markov Decision Process (MDP) formulation of an RL problem. An MDP is defined by the tuple $M = (\mathcal{S}, \mathcal{A}, P, r, \gamma)$, where

\begin{itemize}

\item $\mathcal{S}$ is a (finite or continuous) set of states.

\item $\mathcal{A}$ is a set of actions available to the agent.

\item $P(s'|s,a)$ is the transition probability function (Markovian dynamics), giving the distribution over next states $s'$ when action $a$ is taken in state $s$. \item $r(s,a)$ is a reward function (or $r(s,a,s')$ including next state, depending on context) giving a scalar reward for executing action $a$ in state $s$.

\item $\gamma \in [0,1]$ is a discount factor that weights immediate vs. future rewards (with $\gamma<1$ typically for infinite-horizon problems).

\end{itemize}
\end{definition}

\begin{definition}[Policy and Value Function]
A (stationary) \textbf{policy} $\pi$ is a mapping from states to a distribution over actions. We denote $\pi(a|s)$ as the probability of taking action $a$ in state $s$ under $\pi$.

The value function for a policy $\pi$ is
\[
V^\pi(s) = \mathbb{E}_\pi \left[ \sum_{t=0}^{\infty} \gamma^t r(s_t, a_t) \mid s_0 = s \right],
\]
 the expected cumulative discounted reward starting from state $s$ and following $\pi$. The goal in standard RL is to find an optimal policy $\pi^*$ maximizing $V^{\pi}(s)$ for all $s$ (or maximizing a specific initial state or distribution performance). Equivalently, one maximizes the {\bf return}
 $J(\pi) = \mathbb{E}_{s_0 \sim \rho}\left[ V^\pi(s_0)\right]$
 for some initial state distribution $\rho$. In unconstrained RL, $\pi^*$ solves $\max_{\pi} J(\pi)$.
\end{definition}
 
 \subsection{Constrained Markov Decision Processes (CMDPs)} A Constrained Markov Decision Process extends an MDP with the concept of \emph{costs} (or negative rewards) and associated constraints. 
 Formally, a CMDP can be defined as:
\[
M_C = (S, A, P, r, \{ c^{(i)} \}_{i=1}^{m}, \gamma),
\]
where \( r(s,a) \) is the primary reward as before, and \( c^{(i)}(s,a) \) for \( i=1, \dots, m \) are \( m \) \textbf{cost functions} (or penalty functions)
 encoding the aspects of the task we want to constrain. Each cost function usually corresponds to a particular notion of “unsafe” behavior or resource usage that should be limited. For example, $c^{(1)}(s,a)$ might be an indicator of entering an unsafe state or a measure of damage/risk at state $s$. A policy $\pi$ in a CMDP has not only a reward return $J(\pi) = \mathbb{E}_{\pi}[\sum_t \gamma^t r(s_t,a_t)],$ but also a cost return for each cost function:
\[
J_c^{(i)}(\pi) = \mathbb{E}_{\pi} \left[ \sum_{t=0}^{\infty} \gamma^t c^{(i)}(s_t, a_t) \right].
\]

The safe RL objective can be posed as a {\bf constrained optimization problem} 

\begin{equation} 
\begin{aligned} & \text{maximize}_{\pi} && J(\pi) = \mathbb{E}_\pi\left[\sum{t} \gamma^t r(s_t,a_t)\right], \ & \text{subject to} && J_{c^{(i)}}(\pi) \leq d_i, \quad i = 1,2,\dots,m, 
\end{aligned} 
\label{eq:cmdp_objective} 
\end{equation} 
where $d_i$ is a specified threshold for the $i$-th cost (safety limit). The set
$$\Pi_{\text{safe}} = \{\pi \mid J_{c^{(i)}}(\pi) \le d_i,; \: \forall i \}$$ is called the \textbf{feasible policy set}.

\begin{figure}[t!]
\centering
\begin{tikzpicture}[
  node distance=1.2cm and 2.5cm,
  every node/.style={font=\small},
  block/.style={rectangle, draw=blue!70!black, fill=blue!5, rounded corners=4pt,
    minimum height=1.2cm, minimum width=2.5cm, align=center, line width=0.8pt},
  constraint/.style={rectangle, draw=red!70!black, fill=red!5, rounded corners=4pt,
    minimum height=0.9cm, minimum width=2.8cm, align=center, line width=0.8pt,
    font=\small},
  arr/.style={-{Stealth[length=3mm]}, line width=0.9pt},
  lbl/.style={font=\footnotesize, fill=white, inner sep=1.5pt},
]

\node[block, fill=blue!12] (agent) {\textbf{Agent}\\Policy $\pi(a|s)$};
\node[block, fill=teal!10, right=3.5cm of agent] (env) {\textbf{Environment}\\$P(s'|s,a)$};

\draw[arr, blue!70!black] (agent.north east) -- ++(0,0.5) -| node[lbl, pos=0.25] {Action $a_t$} (env.north west);

\draw[arr, teal!70!black] (env.south west) -- ++(0,-0.5) -| node[lbl, pos=0.25, below] {State $s_{t+1}$} (agent.south east);

\draw[arr, green!60!black] (env.west) -- node[lbl, above] {Reward $r(s,a)$} (agent.east);

\draw[arr, red!70!black] ([yshift=-0.3cm]env.west) -- node[lbl, below] {Cost $c^{(i)}(s,a)$} ([yshift=-0.3cm]agent.east);

\node[constraint, below=1.5cm of $(agent)!0.5!(env)$] (check) {\textbf{Safety Constraint Check}\\$J_{c^{(i)}}(\pi) \leq d_i, \;\forall i$};

\draw[arr, red!50!black, dashed] (agent.south) -- node[lbl, left, xshift=-7pt, font=\scriptsize, align=center] {Evaluate\\policy} (check.north west);
\draw[arr, red!50!black, dashed] (env.south) -- node[lbl, right, xshift=7pt, font=\scriptsize, align=center] {Accumulate\\costs} (check.north east);

\node[below=0.3cm of check, font=\footnotesize\itshape, text=red!60!black] {$\pi \in \Pi_{\text{safe}}$: Feasible policy set};

\end{tikzpicture}
\caption{The Constrained MDP (CMDP) agent-environment interaction loop. In addition to the standard reward signal $r(s,a)$, the environment provides cost signals $c^{(i)}(s,a)$. The agent must find a policy $\pi$ that maximizes cumulative reward while ensuring all cost constraints $J_{c^{(i)}}(\pi) \leq d_i$ are satisfied.}
\label{fig:cmdp_loop}
\end{figure} We assume this set is non-empty (the constraints are attainable). Problem \eqref{eq:cmdp_objective} is the standard formulation of SafeRL as a CMDP optimization problem \cite{Altman1999}. It is a constrained Markov decision problem which, in principle, can be solved via dynamic programming or linear programming if the model is known and state-action spaces are small. Eitan Altman’s foundational work \cite{Altman1999} established that for finite CMDPs, there exists an optimal policy that is stationary (time-independent) and, if multiple constraints are present, possibly stochastic (randomized). Intuitively, sometimes a mixture of actions is required to exactly satisfy multiple constraints: a deterministic policy might violate a constraint, whereas a stochastic policy can blend strategies to meet the constraint bounds exactly.

\paragraph{Lagrangian formulation:} A common theoretical approach to solve CMDPs is to form the Lagrangian of \eqref{eq:cmdp_objective}. Introduce Lagrange multipliers $\lambda = (\lambda_1,\dots,\lambda_m) \ge 0$ for the $m$ constraints. 
The Lagrangian for policy $\pi$ is
\[
\mathcal{L}(\pi, \lambda) = J(\pi) + \sum_{i=1}^{m} \lambda_i \left( d_i - J_c^{(i)}(\pi) \right).
\]

We can rearrange $\mathcal{L}(\pi,\lambda) = J(\pi) - \sum_i \lambda_i J_{c^{(i)}}(\pi) + \sum_i \lambda_i d_i$. Often it is written as $J(\pi) - \sum_i \lambda_i (J_{c^{(i)}}(\pi) - d_i)$ or $J(\pi) - \sum_i \lambda_i J_{c^{(i)}}(\pi)$ up to constants, since $\sum_i \lambda_i d_i$ does not depend on $\pi$. For a fixed $\lambda$, the term 

$$J(\pi) - \sum_i \lambda_i J_{c^{(i)}}(\pi) = \mathbb{E}_{\pi} \left[ \sum_t \gamma^t (r(s_t,a_t) - \sum_i \lambda_i c^{(i)}(s_t,a_t)) \right].$$ 

This suggests defining a \emph{penalized reward} 
$$r_{\lambda}(s,a) = r(s,a) - \sum_{i=1}^m \lambda_i c^{(i)}(s,a).$$ 
For any $\lambda \geq 0$, we can compute
\[
\pi^*(\lambda) = \arg\max_{\pi} \: \mathcal{L}(\pi, \lambda) = \arg\max_{\pi} \: \mathbb{E}_\pi \left[ \sum_{t} \gamma^t r_{\lambda}(s_t, a_t) \right],
\]
which is the optimal policy for the MDP with reward $r_\lambda$. In other words, $\pi^{(\lambda)}$ is the unconstrained optimal policy if we treat $-\lambda_i$ as a weight (penalty) for cost $c^{(i)}$. The dual function is 
\begin{align}
g(\lambda) = \max_\pi \: \mathcal{L}(\pi,\lambda) = J(\pi^{(\lambda)}) + \sum_i \lambda_i(d_i - J_{c^{(i)}}(\pi^{(\lambda)})).
\end{align}
We then minimize $g(\lambda)$ over $\lambda \ge 0$ 
to find the best multipliers
\[
\lambda^* = \arg\min_{\lambda \geq 0} \: g(\lambda).
\]
Under certain conditions (convexity or linearity of the CMDP problem), strong duality holds and solving the dual yields the primal optimum \cite{Altman1999,Achiam2017}. The optimal policy $\pi^*$ for the CMDP is then $\pi^*(\lambda^*)$ (or a mixture of policies if needed when the optimum is not unique). The Lagrangian perspective is very useful in SafeRL for the following reasons
\begin{itemize} 

\item It leads to \textbf{Lagrange multiplier methods} for safe RL, where one maintains estimates of $\lambda_i$ during learning and adjusts them based on constraint violations. Many algorithms in practice (Section 4) use this primal-dual approach. 

\item It gives insight into how costs trade off with reward: $\lambda_i$ can be interpreted as the “price” of violating constraint $i$. A high $\lambda_i$ at optimum means the agent sacrifices a lot of reward to reduce cost $i$. 

\item The gradient of $g(\lambda)$ can be derived as $ \nabla_{\lambda_i} g(\lambda) = d_i - J_{c^{(i)}}(\pi^*(\lambda))$. This leads to a gradient descent update: $\lambda_i \leftarrow \lambda_i + \alpha (J_{c^{(i)}}(\pi) - d_i)$ which intuitively increases the penalty $\lambda_i$ if constraint $i$ is violated ($J_{c^{(i)}} > d_i$) and decreases it if the constraint is satisfied with slack.

\end{itemize}

\begin{figure}[t!]
\centering
\begin{tikzpicture}[
  every node/.style={font=\small},
  block/.style={rectangle, draw=#1!70!black, fill=#1!8, rounded corners=3pt,
    minimum height=1.0cm, minimum width=3.8cm, align=center, line width=0.8pt},
  block/.default=blue,
  arr/.style={-{Stealth[length=3mm]}, line width=0.8pt, #1!70!black},
  arr/.default=gray,
  lbl/.style={font=\footnotesize, inner sep=2pt},
]

\node[block=blue] (policy) at (-3.2, 0) {\textbf{Policy} $\pi_\theta$\\(Primal variable)};
\node[block=teal] (eval) at (3.2, 0) {\textbf{Evaluate}\\$J(\pi_\theta)$, $J_{c^{(i)}}(\pi_\theta)$};
\node[block=purple] (pen) at (-3.2, -3.0) {\textbf{Penalized Reward}\\$r_\lambda = r - \sum_i \lambda_i c^{(i)}$};
\node[block=red] (lambda) at (3.2, -3.0) {\textbf{Multiplier} $\lambda \geq 0$\\(Dual variable)};

\draw[arr=blue] (policy.east) -- node[lbl, above, yshift=4pt] {Roll out episodes} (eval.west);

\draw[arr=red] (eval.south) -- node[lbl, right, font=\scriptsize, align=left] {Constraint gap:\\$J_{c^{(i)}}(\pi_\theta) - d_i$} (lambda.north);

\draw[arr=purple] (lambda.west) -- node[lbl, below, yshift=-4pt, font=\scriptsize] {Penalty weights $\lambda_i$} (pen.east);

\draw[arr=blue] (pen.north) -- node[lbl, left, font=\scriptsize, align=right] {Policy update:\\$\theta \leftarrow \theta + \beta \nabla_\theta J_\lambda(\pi_\theta)$} (policy.south);

\node[font=\scriptsize\bfseries, text=blue!70!black] at (-3.2, 0.85) {Primal: $\max_\theta$};
\node[font=\scriptsize\bfseries, text=red!70!black] at (3.2, -3.85) {Dual: $\min_\lambda$};

\node[font=\scriptsize\itshape, text=red!60!black] at (3.2, -4.25) {$\lambda_i \leftarrow [\lambda_i + \alpha(J_{c^{(i)}} - d_i)]_+$};

\node[font=\footnotesize\itshape, text=gray!60!black] at (0, -5.0) {At convergence: $(\pi^*, \lambda^*)$ solves the CMDP};

\end{tikzpicture}
\caption{The Lagrangian primal-dual optimization framework for CMDPs. The policy $\pi_\theta$ (primal variable) is updated to maximize the penalized reward $r_\lambda$, while the Lagrange multipliers $\lambda$ (dual variables) are updated based on constraint violations. This alternating optimization converges to the CMDP solution under strong duality.}
\label{fig:lagrangian_loop}
\end{figure}

\paragraph{Linear programming solution:} For finite-state CMDPs, an alternative formulation is via occupancy measures and linear programming. One can define variables $x(s,a)$ representing the discounted visitation frequency of state-action pair $(s,a)$ under a stationary policy. The constraints of an optimal occupancy measure include flow conservation (infinite-horizon occupancy distribution) and positivity. 


The total expected discounted reward under a stationary policy $\pi$ is defined as:
\[
J(\pi) = \mathbb{E}_\pi \left[ \sum_{t=0}^\infty \gamma^t r(s_t, a_t) \right].
\]

The occupancy measure $x(s,a)$ represents the discounted visitation frequency of state-action pairs under policy $\pi$
\[
x(s,a) = (1 - \gamma) \sum_{t=0}^\infty \gamma^t \Pr(s_t = s, a_t = a \mid \pi)
\]
which is the expected discounted number of times that the agent visits state $s$ and takes action $a.$

By unrolling the expectation, the expected return can be rewritten in terms of the occupancy measure:
\[
J(\pi) = \sum_{s,a} x(s,a) r(s,a)
\]

This expression is the key to formulating the CMDP as a linear program since the objective becomes linear in $x(s,a)$. Furthermore, the expected cumulative cost constraints can be similarly written as
\[
J_{c^{(i)}}(\pi) = \sum_{s,a} x(s,a) c^{(i)}(s,a) \le d_i, \quad \forall i=1, \ldots, m
\]
which are also linear in $x(s,a)$. This linear structure is crucial as it allows the CMDP optimization problem to be expressed as a linear program (LP).
The CMDP can then be written as a linear program: 
\begin{align*} 
\max_{x(s,a)\ge0} \ & \sum_{s,a} x(s,a) r(s,a) \\
\text{s.t. }\ & \sum_{s,a} x(s,a) c^{(i)}(s,a) \le d_i, \quad i=1,\dots,m, \\
& \sum_{a} x(s,a) = (1-\gamma) \rho(s) + \gamma \sum_{s',a'} P(s|s',a') x(s',a'),  \forall s, \end{align*} 
where $\rho(s)$ is the starting state distribution. This linear program can be solved efficiently for moderate state-action sizes and yields an optimal (potentially stochastic) policy for the CMDP \cite{Altman1999}. While model-based and not directly applicable to large-scale problems, this approach provides theoretical validation that CMDPs are solvable optimally and also serves as a basis for certain planning algorithms in safe RL.

\subsection{Constraint Types and Safety Specifications} The formulation above uses expected cumulative costs as constraints. This is a flexible and popular choice in SafeRL research, but it is worth noting other types of constraints that have been considered 

\begin{enumerate} 

\item \textbf{Instantaneous constraints:} instead of long-term expected cost, one could require $c(s_t,a_t) \le d$ at every time step $t$ (almost surely). This is a stricter requirement (no violations at all). Such hard constraints are challenging for learning, and often enforced via external mechanisms (like safety filters). In CMDP theory, instantaneous constraints can be encoded by making any violation transition to an absorbing failure state with heavy penalty. 

\subsubsection*{Examples of Instantaneous Constraints in Safe RL}

Instantaneous constraints refer to safety requirements that must hold \emph{at every time step} during the agent's execution, rather than only in expectation over a trajectory. Below are typical examples of such constraints arising in real-world applications.

Robotics --Torque or Force Limits \cite{Dalal2018,Cheng2019,Achiam2017}: Robotic manipulators and mobile robots have strict actuator limits. A common constraint is $||\tau_t|| \leq \tau_{\max}$, where $\tau_t$ is the torque vector applied at time $t$. Exceeding these limits even once can cause irreversible hardware damage. Therefore, the torque constraint must hold at every step. Dalal et al.\ \cite{Dalal2018} proposed a safety layer that projects RL actions to satisfy such constraints, while Cheng et al.\ \cite{Cheng2019} combined model-free RL with control barrier functions to enforce actuator limits during learning.

Autonomous Driving --Collision Avoidance \cite{Shalev2017,Isele2018,nguyen2023safe}: Autonomous vehicles must avoid collisions at all times. This is often modeled as a minimum distance constraint, $\text{distance}(s_t) \geq d_{\text{safe}}$, where $d_{\text{safe}}$ is a safety margin. Unlike reward penalties for collisions, instantaneous constraints aim to ensure that no collision ever occurs, even during learning. Shalev-Shwartz et al.\ \cite{Shalev2017} formalized hard safety constraints via the Responsibility-Sensitive Safety (RSS) framework, while Isele et al.\ \cite{Isele2018} used prediction-based constraints to safely learn intersection-handling behaviors.

\begin{figure}[t!]
\centering
\begin{tikzpicture}[
  every node/.style={font=\small},
  arr/.style={-{Stealth[length=2mm]}, line width=0.7pt, #1},
  arr/.default={gray!60!black},
]

\fill[gray!15] (-5.5,-1.2) rectangle (5.5,1.2);
\draw[gray!50, dashed, line width=0.8pt] (-5.5,0) -- (5.5,0);
\draw[gray!70, line width=1.2pt] (-5.5,1.2) -- (5.5,1.2);
\draw[gray!70, line width=1.2pt] (-5.5,-1.2) -- (5.5,-1.2);

\node[rectangle, draw=blue!70!black, fill=blue!20, minimum width=1.4cm,
  minimum height=0.7cm, rounded corners=2pt, line width=0.8pt,
  font=\scriptsize\bfseries] (ego) at (-1.5, -0.5) {Ego};

\node[rectangle, draw=red!70!black, fill=red!15, minimum width=1.4cm,
  minimum height=0.7cm, rounded corners=2pt, line width=0.8pt,
  font=\scriptsize] (other) at (2.5, -0.5) {Vehicle};

\draw[red!60!black, dashed, line width=0.8pt]
  (0.2, -1.05) rectangle (2.5+0.9, 0.05);
\node[font=\tiny, text=red!60!black] at (1.8, -1.3) {Safety zone: $d_{\text{safe}}$};

\draw[{Stealth[length=1.5mm]}-{Stealth[length=1.5mm]}, red!60!black, line width=0.6pt]
  (-0.8, 0.35) -- node[above, font=\tiny, text=red!60!black] {$\text{dist}(s_t) \geq d_{\text{safe}}$} (1.8, 0.35);

\node[rectangle, draw=blue!70!black, fill=blue!10, rounded corners=3pt,
  minimum height=0.8cm, minimum width=2.2cm, align=center, line width=0.7pt,
  font=\scriptsize] (agent) at (-1.5, 2.8) {\textbf{RL Agent} $\pi(a|s)$};

\node[rectangle, draw=teal!70!black, fill=teal!8, rounded corners=3pt,
  minimum height=0.7cm, minimum width=2.0cm, align=center, line width=0.6pt,
  font=\scriptsize] (sensor) at (-1.5, 1.8) {Sensors / Perception};

\node[rectangle, draw=orange!70!black, fill=orange!10, rounded corners=3pt,
  minimum height=0.7cm, minimum width=2.0cm, align=center, line width=0.6pt,
  font=\scriptsize] (shield) at (3.0, 2.8) {Safety Filter};

\node[rectangle, draw=red!60!black, fill=red!8, rounded corners=2pt,
  minimum height=0.6cm, minimum width=2.2cm, align=center, line width=0.6pt,
  font=\tiny] (cst) at (3.0, 1.8) {$c(s_t,a_t) = \mathbf{1}[\text{dist} < d_{\text{safe}}]$};

\draw[arr] (sensor.north) -- (agent.south);
\draw[arr=blue!60!black] (agent.east) -- node[above, font=\tiny] {$a_t$} (shield.west);
\draw[arr=green!50!black] (shield.south) -- node[right, font=\tiny] {$a_t^{\text{safe}}$} (cst.north);
\draw[arr=teal!60!black] (ego.north) -- ++(0,0.35) -| node[pos=0.25, left, font=\tiny] {$s_t$} (sensor.south);

\node[font=\tiny, text=green!50!black, align=center] at (-4.2, 2.8) {Reward:\\reach goal\\quickly};
\node[font=\tiny, text=red!50!black, align=center] at (-4.2, 1.8) {Constraint:\\zero collisions};

\end{tikzpicture}
\caption{SafeRL for autonomous driving. The RL agent receives state observations from sensors and outputs actions (steering, acceleration). A safety filter ensures the executed action satisfies the instantaneous constraint $\text{dist}(s_t) \geq d_{\text{safe}}$, preventing collisions while optimizing travel efficiency.}
\label{fig:app_driving}
\end{figure}

Aerial Vehicles (Drones) --Altitude Constraints \cite{Fisac2019,Gillula2012,Yuan2022}: Drones often operate within restricted altitude corridors, leading to constraints of the form $z_{\min} \leq z_t \leq z_{\max}$. Exceeding altitude boundaries may result in collisions with terrain (if $z_t < z_{\min}$) or violation of airspace regulations (if $z_t > z_{\max}$). Such constraints must be enforced at all times. Fisac et al.\ \cite{Fisac2019} proposed a Hamilton-Jacobi reachability-based framework guaranteeing state constraint satisfaction for quadrotors, and Gillula and Tomlin \cite{Gillula2012} demonstrated guaranteed safe online learning on a quadrotor with altitude bounds.

Medical Applications --Dose Limits in Treatment Planning \cite{Tseng2017,Sprouts2022}: In adaptive radiation therapy or drug administration, instantaneous dosage constraints are essential. The instantaneous constraint may take the form $\text{dose}_t \leq d_{\max}$, limiting the maximum dose administered at each step to prevent severe side effects. Tseng et al.\ \cite{Tseng2017} developed a deep RL framework for dose fractionation in lung cancer constrained by tissue complication limits, and Sprouts et al.\ \cite{Sprouts2022} trained a DRL-based treatment planner enforcing hard dose-volume constraints on organs at risk.

\begin{figure}[t!]
\centering
\begin{tikzpicture}[
  every node/.style={font=\small},
  box/.style={rectangle, draw=#1!70!black, fill=#1!8, rounded corners=3pt,
    minimum height=0.9cm, align=center, line width=0.7pt, font=\scriptsize},
  box/.default=blue,
  arr/.style={-{Stealth[length=2mm]}, line width=0.7pt, #1},
  arr/.default={gray!60!black},
]

\node[box=teal, minimum width=2.2cm] (patient) at (0, 0) {\textbf{Patient State}\\$s_t$: vitals, labs};

\node[box=blue, minimum width=2.4cm] (policy) at (0, 2.5) {\textbf{RL Treatment}\\Policy $\pi(a_t|s_t)$};

\node[box=purple, minimum width=2.2cm] (action) at (4.5, 2.5) {\textbf{Treatment}\\$a_t$: drug dose, \\ventilator settings};

\node[box=red, minimum width=2.4cm] (safety) at (4.5, 0) {\textbf{Safety Constraint}\\$\text{dose}_t \leq d_{\max}$\\$\Pr[\text{mortality}] \leq \delta$};

\node[box=green!50!black, minimum width=2.2cm] (outcome) at (2.25, -2.0) {\textbf{Patient Outcome}\\Recovery / Adverse event};

\draw[arr=blue!60!black] (patient.north) -- node[left, font=\tiny, text=teal!60!black] {observe} (policy.south);
\draw[arr=blue!60!black] (policy.east) -- node[above, font=\tiny] {prescribe $a_t$} (action.west);
\draw[arr=red!60!black] (action.south) -- node[right, font=\tiny, text=red!60!black] {check} (safety.north);
\draw[arr=green!50!black] (safety.south) -- ++(0,-0.5) -| node[pos=0.25, right, font=\tiny, align=left] {safe\\action} ([xshift=0.3cm]outcome.north);

\draw[arr=teal!60!black] ([xshift=-0.3cm]outcome.north) -- ++(0,0.7) -| node[pos=0.25, above, font=\tiny] {$s_{t+1}$} (patient.south);

\node[font=\tiny, text=green!50!black, align=center] at (-2.5, 2.5) {\textbf{Reward:}\\patient recovery\\$r_t = f(\text{vitals})$};
\node[font=\tiny, text=red!50!black, align=center] at (-2.5, 0.8) {\textbf{Cost:}\\adverse effects\\$c_t = g(\text{dose}_t)$};

\draw[arr=red!40!black, dashed] (safety.west) -- node[below, font=\tiny, text=red!50!black] {reject if unsafe} (patient.east);

\end{tikzpicture}
\caption{SafeRL for healthcare treatment planning. The RL policy observes patient state (vitals, lab values) and prescribes treatment actions (drug dosage, ventilator settings). A safety constraint enforces dose limits ($\text{dose}_t \leq d_{\max}$) and bounds mortality risk ($\Pr[\text{mortality}] \leq \delta$). Unsafe actions are rejected and the patient state evolves based on the administered treatment.}
\label{fig:app_healthcare}
\end{figure}

Power Systems --Voltage and Current Limits \cite{Vu2021,Duan2020}: Power grids are subject to operational safety limits such as $V_t \in [V_{\min}, V_{\max}]$ for voltage levels or similar constraints on current. Violations could cause system instability, equipment damage, or even large-scale blackouts. Safe control must respect these constraints instantaneously. Vu et al.\ \cite{Vu2021} proposed barrier function-based safe RL for emergency voltage control with hard safety bounds, while Duan et al.\ \cite{Duan2020} developed a DRL-based autonomous voltage control agent that maintains voltage within operational limits.

Industrial Process Control --Pressure Limits \cite{KimOh2022}: In chemical plants, nuclear reactors, and manufacturing systems, pressure constraints of the form $p_t \leq p_{\max}$ are typical. Exceeding pressure thresholds even once may lead to catastrophic failures such as explosions or hazardous material leaks. Kim and Oh \cite{KimOh2022} developed safe model-based RL using Lyapunov barrier functions for chemical process control (CSTR) with hard state and input constraints on temperature and pressure.

These types of instantaneous constraints are significantly harder to handle than cumulative (long-term) cost constraints since they require the policy to remain within the safe region at every time step, regardless of randomness. In practice, many SafeRL algorithms enforce such constraints through external mechanisms like safety layers, control barrier functions, or shielding.
\item \textbf{Probability of failure:} Here, for example, the constrain is defined as $\Pr(\text{eventual failure}) \le \delta$. If one defines a cost $c(s,a)$ that is $1$ upon entering a failure state and $0$ otherwise, then $J_c(\pi)$ is ``essentially'' the (discounted) probability of failure. A constraint $J_c(\pi)\le \delta$ limits failure probability. This can be handled in CMDP by that cost formulation \cite{Geibel2005}.

\subsubsection*{Examples of Probability of Failure Constraints}

Probability of failure constraints aim to limit the chance that an agent enters a catastrophic or irrecoverable state throughout its lifetime. As discussed, such constraints can be formalized by defining a cost function $c(s,a)$ which equals $1$ when taking an action $a$ in state $s$ leads to a \emph{failure state} (or belongs to a set of failure states), and $0$ otherwise. The expected cumulative cost $J_c(\pi)$ under this formulation directly corresponds to the probability of failure. Below are typical scenarios where such constraints are relevant.

Spacecraft and Autonomous Vehicles --Safe Landing or Docking Probability \cite{Blackmore2010,Ono2015,Chow2015}: In space missions or autonomous landing scenarios, failure is often defined as crashing during landing or docking. One may enforce a constraint such as $\Pr[\text{crash}] \leq \delta$, where $\delta$ is a small acceptable risk level. The cost function is defined as $c(s,a) = 1$ if $(s,a)$ leads to a crash state.

Specifically, The paper \cite{Blackmore2010} is one of the earliest works in chance-constrained motion planning and is frequently cited in spacecraft and UAV planning and \cite{Ono2015}, is a classic reference on chance-constrained formulations for spacecraft landing and docking. The paper \cite{Ono2015} directly deals with probability of failure constraints for spacecraft control.

Robotics: Falling or Tipping Over \cite{Berkenkamp2017,Wabersich2018,Berkenkamp2015}: In humanoid or legged robots, failure is typically associated with falling down. The agent is required to maintain $\Pr[\text{fall}] \leq \delta$ to ensure physical integrity and task feasibility. In this case, any state classified as ``fallen'' is marked as a failure state, and $c(s,a)=1$ if the next state is a fallen state. Specifically, \cite{Berkenkamp2017} is a classic paper that specifically addresses falling in legged robots and balance maintenance as a safety constraint. They model unsafe states (like falls) and ensure with high probability that they are avoided. The paper \cite{Wabersich2018} introduces a safety certification approach ensuring that robots do not enter dangerous states (including falls). It applies to both wheeled and legged robots. The earlier work \cite{Berkenkamp2015} focuses on safe policy learning for balancing and preventing falls.
It explicitly models unsafe states (falls) in the dynamics and safety set. 

Healthcare: Patient Mortality or Critical Failure \cite{Raghu2017,Gottesman2019,jia2020safe,tu2025offline}: In reinforcement learning for clinical decision-making (e.g., ICU treatment policies), a failure may be defined as the patient's mortality or reaching a critical medical condition. The constraint $\Pr[\text{critical\_failure}] \leq \delta$ limits the treatment policy to maintain acceptable risk levels. Here, failure states correspond to medical emergencies.

In particular, \cite{Raghu2017}, models ICU treatment as an MDP where mortality is treated as an absorbing failure state. While optimizing expected return, they explicitly consider trajectories leading to death. In 2018, \cite{otten2023} did a comprehensive study on use of RL systems for ICU treatment. \cite{Gottesman2019}, proposed a foundational paper outlining safety and interpretability concerns in clinical RL. It explicitly discusses mortality and adverse outcomes as critical failure events. While not formalizing constraints as 
$\text{Pr[failure]} \leq \delta,$ it motivates their necessity. 

Finance: Bankruptcy or Insolvency Events \cite{neto2020,Borkar2014,Chow2015,chow2017risk,SCHLOSSER2020104997}: In financial portfolio management, the failure event could be the agent's wealth dropping below a bankruptcy threshold. The probability of this event is often constrained by $\Pr[\text{bankruptcy}] \leq \delta$ to limit risk exposure. The cost function is $c(s,a) = 1$ if wealth crosses the bankruptcy boundary.

The paper \cite{neto2020} discusses risk-sensitive portfolio optimization with Markov decision processes. It addresses ruin (bankruptcy) probabilities explicitly. In 2014 paper \cite{Borkar2014}, proposes to directly deals with probability of ruin (bankruptcy) in constrained MDPs. It proposes algorithms under constraints like 
$\text{Pr[bankruptcy]} \leq \text{Pr[bankruptcy]} \leq \delta.$ The paper \cite{chow2017risk} models percentile-based risk for financial RL tasks where falling below a wealth threshold triggers bankruptcy. While applied to cloud scheduling, demonstrates the same probability of ruin modeling, similar to financial insolvency constraints.

Manufacturing --Production System Breakdown: In industrial automation, a failure might occur when production machinery exceeds thermal, mechanical, or chemical safety limits leading to breakdown. The probability of such system failure is constrained to be below a pre-specified threshold, e.g., $\Pr[\text{breakdown}] \leq \delta$.

Power Grids --Blackout Events: In power system control, blackouts (large-scale power failures) are often modeled as absorbing failure states. The system may enforce $\Pr[\text{blackout}] \leq \delta$ to reduce the chance of a cascading failure. Failure is usually caused by overloading, component failures, or instability.

In all these scenarios, $J_c(\pi)$ acts as the failure probability and CMDPs provide a natural framework for enforcing such probabilistic constraints.
\item \textbf{Risk measures:} Instead of expectation of cumulative cost, one could constrain a risk measure of the return (or cost). For example, constrain the variance of return below a threshold, or ensure CVaR$_\alpha$(cost) $\le \delta$. Some works incorporate CVaR into RL as a way to ensure low probability of catastrophic outcomes \cite{Chow2015}. These constraints often do not fit the linear structure of CMDPs, but can be tackled with specialized algorithms. 

\subsubsection*{Examples of Risk Measure Constraints}

Risk measure constraints go beyond the expectation of cumulative cost and aim to control higher-order statistics or tail behavior of the cost distribution. These constraints are useful when we are concerned not only with average performance but also with rare but high-impact events. The most common risk measures in SafeRL include variance, Value-at-Risk (VaR), and Conditional Value-at-Risk (CVaR). Below are several examples from real-world applications.

Autonomous Driving --Variance-Constrained Driving Comfort: While avoiding collisions is a safety constraint, maintaining comfortable driving also involves controlling the variance of acceleration, jerk, or lane deviations. A variance constraint of the form $\text{Var}\left[\sum_t c(s_t,a_t)\right] \leq \delta$ can ensure that passenger discomfort due to aggressive or unstable maneuvers remains limited, reducing the risk of loss of control or accidents.

Specifically, in 2012, \cite{Tamar2012}, introduced variance-constrained reinforcement learning where variance of return is explicitly controlled. It is applicable to driving scenarios when controlling variance of acceleration or jerk. \cite{huang2026} explicitly focuses on reducing control variability to improve smoothness and driving comfort. In 2021 \cite{Kiran2021} wrote a comprehensive survey that discusses driving comfort (acceleration, jerk minimization, smoothness) as key secondary objective in autonomous driving, and mentions various papers that incorporates constraints. 

Finance: CVaR-Constrained Portfolio Optimization: In financial portfolio management, it is common to limit the Conditional Value-at-Risk (CVaR) of the portfolio's return. A CVaR constraint $\text{CVaR}_\alpha[\text{loss}] \leq \delta$ ensures that the expected loss in the worst $\alpha\%$ of cases does not exceed a tolerable threshold. This is widely used to manage downside risk beyond what variance alone captures.

A seminal paper on optimization of conditional Value-at-Risk was by \cite{Rockafellar2000} for portfolio problems. In 2014, \cite{Prashanth2014} specifically focuses on CVaR-constrained MDPs with application to financial risk. This is the go-to reference in both optimization and financial risk management. In 2015, \cite{Chow2015}, introduces CVaR-constrained RL applicable to portfolio optimization and other decision-making tasks. It provides methods to enforce CVaR constraints in sequential decision-making.
 In the same year \cite{Tamar2015}, introduces policy gradient methods for risk-sensitive criteria including CVaR. It directly applies to portfolio optimization under CVaR constraints.

\begin{figure}[t!]
\centering
\begin{tikzpicture}[
  every node/.style={font=\small},
  box/.style={rectangle, draw=#1!70!black, fill=#1!8, rounded corners=3pt,
    minimum height=0.85cm, align=center, line width=0.7pt, font=\scriptsize},
  box/.default=blue,
  arr/.style={-{Stealth[length=2mm]}, line width=0.7pt, #1},
  arr/.default={gray!60!black},
]

\node[box=teal, minimum width=2.4cm] (market) at (0, 0) {\textbf{Market}\\$s_t$: prices, indicators};

\node[box=blue, minimum width=2.6cm] (agent) at (0, 2.5) {\textbf{RL Portfolio Agent}\\$\pi(a_t|s_t)$};

\node[box=purple, minimum width=2.4cm] (alloc) at (4.8, 2.5) {\textbf{Allocation}\\$a_t$: portfolio weights\\$w_1, w_2, \ldots, w_n$};

\node[box=red, minimum width=2.6cm] (cvar) at (4.8, 0) {\textbf{Risk Constraint}\\$\text{CVaR}_\alpha[\text{loss}] \leq \delta$};

\node[box=green!50!black, minimum width=2.2cm] (ret) at (2.4, -2.0) {\textbf{Portfolio Returns}\\$r_t = \sum_i w_i \cdot R_i$};

\begin{scope}[shift={(8.0, 0.5)}, scale=0.6]
  \draw[-{Stealth[length=1.5mm]}, gray!60!black, line width=0.5pt] (-2.5,0) -- (2.5,0)
    node[right, font=\tiny] {Loss};
  \draw[-{Stealth[length=1.5mm]}, gray!60!black, line width=0.5pt] (-2.5,0) -- (-2.5,2.5)
    node[above, font=\tiny] {Prob};
  \draw[blue!60!black, line width=0.8pt, smooth]
    plot[domain=-2.2:2.2, samples=40] (\x, {2.0*exp(-(\x+0.3)*(\x+0.3)/0.8)});
  \fill[red!20] plot[domain=0.8:2.2, samples=20] (\x, {2.0*exp(-(\x+0.3)*(\x+0.3)/0.8)})
    -- (2.2,0) -- (0.8,0) -- cycle;
  \draw[red!60!black, line width=0.7pt, dashed] (0.8,0) -- (0.8,1.0);
  \node[font=\tiny, text=red!60!black] at (0.4, -0.5) {VaR$_\alpha$};
  \node[font=\tiny, text=red!60!black] at (2.0, -0.5) {CVaR$_\alpha$};
  \draw[-{Stealth[length=2.5mm]}, red!60!black, line width=1.0pt] (1.5, 2.0) -- (1.5, 0.15);
  \node[font=\tiny\bfseries, text=red!60!black, right] at (1.6, 1.8) {Tail risk};
\end{scope}

\draw[arr=blue!60!black] (market.north) -- node[left, font=\tiny, text=teal!60!black] {observe} (agent.south);
\draw[arr=blue!60!black] (agent.east) -- node[above, font=\tiny] {allocate} (alloc.west);
\draw[arr=red!60!black] (alloc.south) -- node[right, font=\tiny, text=red!60!black] {risk check} (cvar.north);
\draw[arr=green!50!black] (cvar.south) -- ++(0,-0.5) -| node[pos=0.25, right, font=\tiny, align=left] {execute\\trade} ([xshift=0.5cm]ret.north);

\draw[arr=teal!60!black] ([xshift=-0.5cm]ret.north) -- ++(0,0.3) -| node[pos=0.75, left, font=\tiny] {$s_{t+1}$} (market.south);

\node[font=\tiny, text=green!50!black, align=center] at (-2.8, 2.5) {\textbf{Reward:}\\maximize\\expected return};
\node[font=\tiny, text=red!50!black, align=center] at (-2.8, 1.2) {\textbf{Cost:}\\tail losses\\$\Pr[\text{ruin}] \leq \delta$};

\end{tikzpicture}
\caption{SafeRL for risk-constrained portfolio management. The RL agent observes market state and outputs portfolio allocations. A CVaR constraint $\text{CVaR}_\alpha[\text{loss}] \leq \delta$ limits tail risk, ensuring the expected loss in the worst $\alpha\%$ of scenarios stays bounded. The inset shows the loss distribution with the CVaR tail region highlighted.}
\label{fig:app_finance}
\end{figure}

Healthcare: CVaR for Adverse Outcomes: In healthcare applications such as treatment planning or resource allocation, minimizing the expected number of adverse events may not be sufficient. A CVaR constraint on cumulative adverse events or side effects ensures that treatment policies control the likelihood of rare but severe negative outcomes.

Although general, the paper \cite{Prashanth2016} is frequently cited in healthcare RL as it provides risk-sensitive methods including CVaR for controlling adverse events. A position paper \cite{Gottesman2019} emphasizes that minimizing expected adverse outcomes is insufficient and recommended the use of risk measures (e.g., CVaR). Although it doesn't present an algorithm, it motivates CVaR as a necessary tool in treatment planning.

To summarize, CVaR is used in healthcare RL to limit the risk of rare but severe adverse events, to model tail risk (e.g., mortality, critical organ failure, side effects), and to design safe treatment policies under uncertainty.

Supply Chain Management --Risk-Averse Inventory Control: In supply chains, stockout events (inventory drops below zero) cause disruptions. Instead of just minimizing expected stockouts, a CVaR constraint on stockout penalties ensures that even in rare demand spikes, the risk of large cumulative stockouts is controlled.

One of the foundational works on risk-averse inventory control, discusses CVaR and other risk measures \cite{Chen2014}. It models stockouts as undesirable events and controls the risk via dynamic programming. Widely used as a textbook, \cite{Shapiro2014} includes detailed treatment of CVaR in supply chain optimization. It explains how risk measures such as CVaR can control stockouts and demand uncertainties. While not supply chain specific, \cite{Chow2015} is often cited in supply chain literature for inventory control under CVaR constraints. It's techniques directly apply when modeling stockouts as risky events. A highly cited paper is \cite{Bertsimas2006} showing how robust optimization (a precursor to CVaR-type models) controls stockout risks. Provides insight into handling demand uncertainty and stockout penalties.

Robotics --CVaR-Constrained Trajectory Optimization: For autonomous robots navigating uncertain environments, one may use CVaR constraints on cumulative collision risk or energy consumption. This ensures that the robot does not just minimize average risk but is also robust against worst-case environmental uncertainties or adversarial perturbations.

The paper \cite{Ahmadi2022} directly studies CVaR-constrained trajectory optimization for robots under uncertainty (see also \cite{Hakabyan2019,ahmadi2020risk,Bian2023} for CVaR-based motion planning and path planning). Formulates trajectory optimization problems where collision risk is controlled using CVaR. The paper \cite{Zhang2020} focuses on risk-averse path planning under environmental uncertainty using CVaR. Provides algorithms and examples for safe robot navigation with collision risk control. 

Power Systems --Risk-Sensitive Stability Control: In power grid operations, rather than just minimizing expected frequency deviations or power outages, operators may use CVaR or variance constraints to ensure that the probability of large-scale instabilities remains acceptably low, accounting for rare but impactful demand or supply fluctuations.

The paper \cite{Bitar2012} addresses reliability and demand uncertainties in power systems with a risk-sensitive approach. It Models constraints on load shedding and supply-demand balancing. The paper \cite{DallAnese2015} directly introduces chance-constrained optimization for voltage stability and power flow \cite{zhang2011chance}, which is equivalent to controlling the probability of instability; CVaR and probability bounds are discussed. In 2012, \cite{roald2014} introduces risk-constrained OPF formulations using CVaR and chance constraints. It ensures that the probability of voltage violations and instabilities is below a prescribed risk level. A year earlier in 2011, \cite{Wang2025ProbabilityLF} models the variance and tail risk of power system instability due to fluctuating wind generation.
While {\bf not} CVaR directly, it motivates variance and higher-order moment-based risk constraints.

Power systems use variance, probabilistic, and CVaR constraints to: Avoid rare but catastrophic blackouts, to maintain voltage and frequency within safe margins, and to ensure reliability under demand and renewable generation uncertainty.

Risk measures provide a flexible modeling tool for specifying safety, robustness, and fairness. However, incorporating them often requires non-standard methods such as CVaR-optimized policy gradients, distributional reinforcement learning, or scenario-based optimization, as these constraints typically violate the linear structure required for classic CMDP formulations.
\item \textbf{Multi-objective viewpoint:} SafeRL can be seen as a multi-objective optimization where one objective is reward and others are (negative) costs \cite{Horie2019,gu2025safe}. The constraint formulation picks one point on the Pareto frontier by treating costs as hard constraints. Alternatively, one could combine reward and costs into a single scalar reward via weighted sum (penalty method), but that requires tuning weights and does not guarantee constraint satisfaction \cite{Achiam2017}. Constrained formulation cleanly separates objectives and safety. 

\subsubsection*{Examples of Multi-Objective Viewpoint in Safe Reinforcement Learning}

In many real-world applications, agents must simultaneously optimize multiple objectives that may conflict. Typically, SafeRL is modeled as a multi-objective problem where one objective is the primary reward, while others are safety-related costs. The CMDP formulation addresses this by enforcing costs as hard constraints, selecting a specific point on the Pareto frontier. Alternatively, some works use a scalarization (penalty) method by combining reward and costs into a single objective. Below are common examples illustrating the multi-objective viewpoint.

Robotics --Speed vs. Safety Trade-off \cite{Achiam2017,Berkenkamp2017,Chow2018}: A mobile robot navigating in an environment may aim to maximize the reward associated with reaching the goal quickly. However, it also needs to minimize the probability of collisions and energy consumption. Here, speed contributes positively to the reward, while collisions and energy usage are treated as negative costs. The CMDP formulation could enforce a maximum acceptable collision rate and energy budget, leading to an explicit safety-performance trade-off.

Autonomous Driving --Travel Time vs. Accident Risk \cite{Dalal2018,Shalev2017,Zheng2024,nguyen2023safe}: In autonomous driving, agents aim to minimize the expected travel time while simultaneously ensuring a low probability of accidents. The agent faces a trade-off between driving faster (leading to higher reward) and maintaining safe distances or reduced speeds to avoid collisions (cost). The Pareto frontier consists of policies ranging from conservative (low accident risk, long travel time) to aggressive (low travel time, high accident risk). SafeRL selects a policy on this frontier according to the safety constraint.

Energy Systems --Power Supply vs. Cost and Reliability \cite{Bitar2012,DallAnese2015}: In power grid management, the agent may aim to optimize electricity production to meet demand (reward) while minimizing costs associated with fuel consumption and the risk of violating reliability standards (costs). This problem naturally involves multiple objectives: maximizing supply quality and minimizing operational risks and costs.

Healthcare: Treatment Success vs. Adverse Effects \cite{Gottesman2019,Raghu2017}: In medical decision-making, an RL agent may need to maximize treatment efficacy while minimizing adverse effects or treatment toxicity. For example, maximizing patient recovery speed could conflict with the need to limit drug dosage to avoid harmful side effects. A CMDP constraint could limit the expected cumulative adverse effects to a tolerable threshold, enforcing safety.

Manufacturing: Production Efficiency vs. Maintenance Costs \cite{Chung2020,Siraskar2023,CHEN2025111018}: In automated manufacturing, increasing production speed or output (reward) may result in higher machine wear and maintenance costs (costs). A CMDP-based SafeRL framework may impose a constraint on expected maintenance cost or machine degradation, forcing the agent to balance throughput and longevity.

Drone Swarms: Task Completion vs. Communication Load \cite{Gu2023MultiRobot}: In multi-drone systems, agents may wish to maximize task completion rates (reward) while minimizing communication overhead (cost). Communication constraints can act as safety constraints in environments with bandwidth limitations or interference risks.

In all these cases, treating costs as hard constraints via CMDPs gives a systematic way to trade off reward and cost by directly selecting a feasible point on the Pareto frontier. In contrast, using a scalarization approach (reward minus weighted costs) can lead to policies that violate constraints unless the weights are carefully chosen and tuned.

\item \textbf{Temporal logic specifications \cite{alshiekh2018safe,ElSayed2021,Turchetta2016,Wabersich2018}:} In some safety-critical settings, the safety requirement is given as a formal temporal logic formula (e.g., “always avoid region X unless Y happens”). Such logic specifications can be converted to automata and then to reward/cost functions or shields that enforce them \cite{alshiekh2018safe,ElSayed2021}. While not a traditional CMDP constraint, they can often be incorporated by extending the state space to include automaton states representing the satisfaction of the formula. Specifically, in \cite{alshiekh2018safe}, Alshiekh et al. (2018)
introduced safe reinforcement learning via shielding; they used LTL (Linear Temporal Logic) specifications to construct shields for RL agents. In Wabersich and Zeilinger (2018) \cite{Wabersich2018}, linear model predictive safety certification for learning-based control was employed. Although it focused on model predictive safety, their framework is capable of incorporating logic-based safety constraints. The paper \cite{ElSayed2021} extends shield-based safe RL to the multi-agent setting using temporal logic specifications. The paper \cite{Turchetta2016} while focused on safe exploration, their work shows how formal safety specifications can be integrated into exploration guarantees. One of the older but influential paper \cite{Sadigh2016} shows how specifications from temporal logic can shape safe planning.

\end{enumerate}

\begin{figure}[t!]
\centering
\begin{tikzpicture}[
  every node/.style={font=\small},
  hub/.style={rectangle, draw=blue!70!black, fill=blue!15, rounded corners=5pt,
    minimum height=1.2cm, minimum width=2.8cm, align=center, line width=1pt,
    font=\small\bfseries},
  app/.style={rectangle, draw=#1!70!black, fill=#1!8, rounded corners=3pt,
    minimum height=0.85cm, minimum width=2.5cm, align=center, line width=0.7pt,
    font=\scriptsize},
  app/.default=gray,
  ctype/.style={font=\tiny, text=#1!60!black, align=center},
  ctype/.default=gray,
  arr/.style={-{Stealth[length=2mm]}, line width=0.7pt, #1!50!black},
  arr/.default=gray,
]

\node[hub] (hub) at (0,0) {Safe RL\\(CMDP)};

\node[app=red] (drive) at (-3.5, 2.5) {\textbf{Autonomous}\\\textbf{Driving}};
\node[app=teal] (health) at (0, 3.0) {\textbf{Healthcare}};
\node[app=orange] (finance) at (3.5, 2.5) {\textbf{Finance}};

\node[app=purple] (robot) at (-3.5, -2.5) {\textbf{Robotics}};
\node[app=green!50!black] (power) at (0, -3.0) {\textbf{Power}\\\textbf{Systems}};
\node[app=brown] (mfg) at (3.5, -2.5) {\textbf{Manufacturing}};

\draw[arr=red] (hub.north west) -- (drive.south east);
\draw[arr=teal] (hub.north) -- (health.south);
\draw[arr=orange] (hub.north east) -- (finance.south west);
\draw[arr=purple] (hub.south west) -- (robot.north east);
\draw[arr=green!50!black] (hub.south) -- (power.north);
\draw[arr=brown] (hub.south east) -- (mfg.north west);

\node[ctype=red, above=0.1cm of drive] {$\text{dist}(s_t) \geq d_{\text{safe}}$};
\node[ctype=teal, above=0.1cm of health] {$\text{dose}_t \leq d_{\max}$};
\node[ctype=orange, above=0.1cm of finance] {$\text{CVaR}_\alpha \leq \delta$};
\node[ctype=purple, below=0.1cm of robot] {$\|\tau_t\| \leq \tau_{\max}$};
\node[ctype=green!50!black, below=0.1cm of power] {$V_t \in [V_{\min}, V_{\max}]$};
\node[ctype=brown, below=0.1cm of mfg] {$\Pr[\text{fail}] \leq \delta$};

\end{tikzpicture}
\caption{Application domains of Safe Reinforcement Learning. Each domain connects to the CMDP framework through domain-specific safety constraints: collision avoidance in autonomous driving, dosage limits in healthcare, risk measures (CVaR) in finance, actuator limits in robotics, operational bounds in power systems, and failure probability in manufacturing.}
\label{fig:app_domains}
\end{figure}

Throughout this survey, we largely assume the standard expected cumulative cost constraints unless stated otherwise. This assumption covers many practical cases (like average constraint violation rate, or total resource consumption) and has well-developed theoretical tools. When discussing specific algorithms, we will note what type of constraint they handle (most often, it is expected cost).

\begin{figure}[t!]
\centering
\begin{tikzpicture}[
  every node/.style={font=\scriptsize},
  catbox/.style={rectangle, draw=#1!70!black, fill=#1!10, rounded corners=3pt,
    minimum height=0.9cm, minimum width=2.4cm, align=center, line width=0.7pt},
  catbox/.default=blue,
  topbox/.style={rectangle, draw=blue!70!black, fill=blue!15, rounded corners=4pt,
    minimum height=0.9cm, minimum width=4.0cm, align=center, line width=0.9pt,
    font=\small\bfseries},
  descbox/.style={font=\tiny, text=gray!70!black, align=center, text width=2.4cm},
  arr/.style={-{Stealth[length=2mm]}, line width=0.6pt, gray!60!black},
]

\node[topbox] (top) at (0,0) {Safety Constraint Types in SafeRL};

\node[catbox=orange] (inst)  at (-5.6,-1.8) {\textbf{Instantaneous}\\$c(s_t,a_t) \leq d,\;\forall t$};
\node[catbox=teal]   (cumul) at (-2.8,-1.8) {\textbf{Expected}\\  \textbf{Cumulative}\\$J_c(\pi) \leq d$};
\node[catbox=red]    (prob)  at (0,-1.8)    {\textbf{Probability of}\\  \textbf{Failure}\\$\Pr[\text{fail}] \leq \delta$};
\node[catbox=purple] (risk)  at (2.8,-1.8)  {\textbf{Risk Measures}\\CVaR, Variance};
\node[catbox=cyan]   (ltl)   at (5.6,-1.8)  {\textbf{Temporal Logic}\\LTL specifications};

\draw[arr] (top.south) -- ++(0,-0.25) -| (inst.north);
\draw[arr] (top.south) -- ++(0,-0.25) -| (cumul.north);
\draw[arr] (top.south) -- ++(0,-0.25) -| (prob.north);
\draw[arr] (top.south) -- ++(0,-0.25) -| (risk.north);
\draw[arr] (top.south) -- ++(0,-0.25) -| (ltl.north);

\node[descbox, below=0.15cm of inst]  {Hard per-step limits.\\Torque, collision avoid.};
\node[descbox, below=0.15cm of cumul] {Long-run average cost.\\Most common in CMDPs.};
\node[descbox, below=0.15cm of prob]  {Bound on catastrophic\\event probability.};
\node[descbox, below=0.15cm of risk]  {Tail risk via CVaR$_\alpha$\\or variance bounds.};
\node[descbox, below=0.15cm of ltl]   {Formal specs via\\automata and shields.};

\draw[{Stealth[length=2mm]}-{Stealth[length=2mm]}, line width=0.5pt, orange!60!black, dashed]
  (-6.2,-4.0) -- node[below, font=\scriptsize\itshape] {Strictest $\longleftrightarrow$ Most flexible} (6.2,-4.0);

\end{tikzpicture}
\caption{Taxonomy of safety constraint types in SafeRL. Instantaneous constraints are the strictest (must hold at every time step), while expected cumulative constraints (the standard CMDP formulation) are the most common. Probability of failure, risk measures, and temporal logic specifications offer alternative ways to encode safety requirements.}
\label{fig:constraint_types}
\end{figure}

\subsection{Theoretical Results} We highlight a few key theoretical results for CMDPs relevant to SafeRL: 

\begin{theorem}[{\bf Optimal Policy for CMDP \cite{Altman1999}}] 
For a finite CMDP with bounded rewards and costs, there exists an optimal policy $(\pi^*, \lambda^*)$ that attains the maximum in \eqref{eq:cmdp_objective} (and corresponding optimal dual variables). Moreover, there exists an optimal policy that is stationary (time-independent) and can be chosen to be deterministic with respect to actions at all but possibly a measure-zero set of states. In practice, optimal policies may randomize between a small number of deterministic policies if needed to exactly satisfy constraints. 
\end{theorem} 

In short, one does not need complex history-dependent or non-Markovian policies to solve CMDPs optimally; memoryless policies suffice, simplifying the search space for algorithms.

\begin{theorem}[{\bf Lagrange Duality \cite{Altman1999}}] 
Under mild regularity conditions (e.g., finite state/action or convexity in policy space), The strong duality holds for the CMDP problem. That is,
\[
\min_{\lambda \geq 0} \max_{\pi} \mathcal{L}(\pi, \lambda) = \max_{\pi} \min_{\lambda \geq 0} L(\pi, \lambda),
\]
and solving the dual yields the primal optimum. The optimal dual variables $\lambda^*$ provide valuable information: if $\lambda_i^* > 0$, then the $i$-th constraint is active (tight) at the optimum; if $\lambda_i^* = 0$, the optimum policy naturally satisfies $i$-th constraint with some slack. 
\end{theorem} 

This theorem justifies many SafeRL approaches that focus on solving the dual via gradient methods on $\lambda$ while finding optimal policies for a given $\lambda$ using RL.

\begin{proposition}[{\bf Policy Gradient for Constrained Objectives}] If the policy $\pi_\theta$ is parameterized by $\theta$ (e.g., a neural network), one can derive gradients for the constrained problem. For instance, using the Lagrangian, the gradient of $\mathcal{L}(\pi_\theta,\lambda)$ with respect to $\theta$ is
\[
\nabla_{\theta} \mathcal{L} = \nabla_{\theta} J(\pi_{\theta}) - \sum_{i} \lambda_i \nabla_{\theta} J_c^{(i)}(\pi_{\theta}).
\]
This leads to constrained policy gradient algorithms, where $\theta$ is updated in the direction of $\nabla_\theta \mathcal{L}$ and $\lambda$ is updated in the direction of $\nabla_\lambda \mathcal{L} = d_i - J_{c^{(i)}}(\pi_\theta)$. Many actor-critic style SafeRL methods employ this simultaneous gradient update (a form of primal-dual gradient descent) \cite{Chow2018}. 
\end{proposition}

\begin{proposition}[{\bf Policy Performance Bounds~\cite{Achiam2017}}]
For any two policies $\pi$ and $\pi'$, let $J(\pi)$ denote the expected reward return,
and $J_{C_i}(\pi)$ denote the expected return for a cost function $C_i$.
The change in performance when updating from $\pi$ to $\pi'$ is bounded as (as given by Corollary 1 and Corollary 2 of \cite{Achiam2017}).

\noindent\textbf{Reward (Lower Bound):}
The improvement in expected reward is bounded by
\begin{align}
J(\pi') - J(\pi)
&\ge
\frac{1}{1 - \gamma}
\mathbb{E}_{s \sim d^{\pi},\, a \sim \pi'} \Big[
A^{\pi}(s,a) \nonumber \\
&\quad
- \frac{2 \gamma \epsilon^{\pi'}}{1 - \gamma}
D_{TV}(\pi' \| \pi)[s]
\Big].
\label{eq:reward_bound}
\end{align}

\noindent\textbf{Cost (Upper Bound):}
The change in expected cost is bounded by
\begin{align}
J_{C_i}(\pi') - J_{C_i}(\pi)
&\le
\frac{1}{1 - \gamma}
\mathbb{E}_{s \sim d^{\pi},\, a \sim \pi'} \Big[
A^{\pi}_{C_i}(s,a) \nonumber \\
&\quad
+ \frac{2 \gamma \epsilon^{\pi'}_{C_i}}{1 - \gamma}
D_{TV}(\pi' \| \pi)[s]
\Big].
\label{eq:cost_bound}
\end{align}

\noindent
Here, $A^{\pi}$ and $A^{\pi}_{C_i}$ are the advantage functions for the reward
and cost, respectively; $d^{\pi}$ is the state distribution of policy $\pi$;
the $\epsilon$ terms represent the maximum absolute advantage values;
and $D_{TV}$ denotes the Total Variation divergence.
\end{proposition}

\begin{proof}[Sketch of proof, detail in \cite{Achiam2017}]
\textbf{(1) Reward-shaping identity and surrogate.}
For any $f:S\to\mathbb{R}$ define
$\delta_f(s,a,s') := R(s,a,s')+\gamma f(s')-f(s)$ and
\[
L_{\pi,f}(\pi') :=
\mathbb{E}_{s\sim d^{\pi},\,a\sim\pi',\,s'\sim P}\!\Big[
\big(\tfrac{\pi'(a|s)}{\pi(a|s)}-1\big)\,\delta_f(s,a,s')
\Big].
\]
Using the discounted visitation measures one shows
\[
J(\pi')-J(\pi)
= \frac{1}{1-\gamma}\!\left(
\mathbb{E}_{d^{\pi'}}[\delta_f]-\mathbb{E}_{d^{\pi}}[\delta_f]\right),
\]
and by adding/subtracting $\langle d^{\pi},\overline{\delta}^{\pi'}_f\rangle$
and applying H\"{o}lder's inequality, \cite{Achiam2017} obtain the two-sided bound
\begin{align*}
\frac{1}{1\!-\!\gamma}\!\left(L_{\pi,f}(\pi') - 2\,\|\pi'\|_{f}\,D_{TV}(d^{\pi'}\!\|\!d^{\pi})\right)
&\le J(\pi')-J(\pi) \\
&\le
\frac{1}{1\!-\!\gamma}\!\left(L_{\pi,f}(\pi') + 2\,\|\pi'\|_{f}\,D_{TV}(d^{\pi'}\!\|\!d^{\pi})\right),
\end{align*}
where $\|\pi'\|_{f} := \max_{s}\big|\mathbb{E}_{a\sim\pi'}[\delta_f(s,a,s')]\big|$.

\textbf{(2) From $D_{TV}(d^{\pi'}\!\|\!d^{\pi})$ to statewise TV between policies.}
Bound the shift in discounted visitation by the average per-state TV between action distributions, yielding
\[
D_{TV}(d^{\pi'}\!\|\!d^{\pi})
\;\le\; \tfrac{\gamma}{1-\gamma}\,
\mathbb{E}_{s\sim d^{\pi}}\!\big[D_{TV}(\pi'\!\|\!\pi)[s]\big].
\]
Substituting this into the previous display gives the Theorem~1 bounds in \cite{Achiam2017}.

\textbf{(3) Choose $f$ to recover advantages.}
Setting $f=V^{\pi}$ makes $\mathbb{E}_{s'\sim P}[\delta_f|s,a]=A^{\pi}(s,a)$, and
$\|\pi'\|_{f}$ becomes $\epsilon^{\pi'}:=\max_s|\mathbb{E}_{a\sim\pi'}[A^{\pi}(s,a)]|$, yielding
\[
J(\pi')-J(\pi)\;\ge\;
\frac{1}{1-\gamma}\,
\mathbb{E}_{s\sim d^{\pi},\,a\sim\pi'}\!\Big[
A^{\pi}(s,a)
- \frac{2\gamma\,\epsilon^{\pi'}}{1-\gamma}\,D_{TV}(\pi'\!\|\!\pi)[s]
\Big],
\]
which is Eq.~\eqref{eq:reward_bound} (lower bound).
Likewise, taking $f=V^{\pi}_{C_i}$ produces the \emph{cost} version with
$A^{\pi}_{C_i}$ and $\epsilon^{\pi'}_{C_i}$, giving Eq.~\eqref{eq:cost_bound} (upper bound).
\end{proof}

These bounds are significant as they justify using the expected advantage
(the first term inside the expectation) as a surrogate objective for policy optimization.
The bounds formally characterize the worst-case approximation error
(the second term) that arises from using the state distribution $d^{\pi}$ of the old policy
instead of the new policy's distribution $d^{\pi'}$.

\paragraph{Safe exploration and probably safe learning:} A distinction in SafeRL theory is between methods that guarantee \emph{safety during learning} vs. only \emph{at convergence}. Most theoretical results (like the ones above) ensure that the final learned policy can satisfy constraints. Ensuring that intermediate policies (during training) also satisfy constraints is much harder. Constrained policy optimization approaches (Section 4) aim to maintain safety at each iteration by conservative updates \cite{Achiam2017}. Another line of work uses PAC-style analysis or high-probability bounds to derive exploration strategies that with high probability never violate constraints beyond a tolerance \cite{Turchetta2016,Berkenkamp2017}. These often rely on optimistic models or Lyapunov functions to formally verify safe regions of state-space the agent can explore. Though we do not delve into detailed proofs, we note that providing safety guarantees during learning typically requires additional assumptions (such as mild system dynamics, or an initial safe policy to bootstrap from). Having established the CMDP framework and theoretical background, we now move on to discuss concrete algorithms and methods developed for SafeRL, both in the single-agent case (Section 4) and multi-agent extensions (Section 5).

\section{State-of-the-Art Methods in SafeRL and SafeMARL} In this section, we survey major methods and algorithms in Safe Reinforcement Learning, covering both single-agent SafeRL in CMDP settings and extensions to multi-agent SafeMARL. We organize the discussion by methodological categories, explaining how each approach incorporates safety and highlighting key algorithms. For each category, we provide examples of state-of-the-art techniques and cite representative works.

\begin{figure}[t!]
\centering
\begin{tikzpicture}[
  every node/.style={font=\scriptsize},
  root/.style={rectangle, draw=blue!70!black, fill=blue!15, rounded corners=4pt,
    minimum height=0.8cm, minimum width=3.5cm, align=center, line width=0.9pt,
    font=\footnotesize\bfseries},
  cat/.style={rectangle, draw=#1!70!black, fill=#1!10, rounded corners=3pt,
    minimum height=0.7cm, minimum width=2.2cm, align=center, line width=0.7pt,
    font=\scriptsize\bfseries},
  cat/.default=teal,
  method/.style={rectangle, draw=gray!50, fill=gray!5, rounded corners=2pt,
    minimum height=0.45cm, minimum width=2.2cm, align=center, font=\scriptsize,
    line width=0.5pt},
  arr/.style={-{Stealth[length=2mm]}, line width=0.6pt, gray!60!black},
]

\node[root] (root) at (0,0) {SafeRL \& SafeMARL Methods};

\node[cat=orange] (lagr)   at (-4.5,-1.5) {Constrained\\[-1pt]Optimization};
\node[cat=teal]   (shield) at (-1.5,-1.5) {Safety Shields};
\node[cat=purple] (risk)   at (1.5,-1.5)  {Risk-Sensitive};
\node[cat=cyan]   (marl)   at (4.5,-1.5)  {Multi-Agent\\[-1pt]Extensions};

\draw[arr] (root.south) -- ++(0,-0.2) -| (lagr.north);
\draw[arr] (root.south) -- ++(0,-0.2) -| (shield.north);
\draw[arr] (root.south) -- ++(0,-0.2) -| (risk.north);
\draw[arr] (root.south) -- ++(0,-0.2) -| (marl.north);

\node[method] (m1) at (-4.5,-2.6) {Lagrange Actor-Critic};
\node[method] (m2) at (-4.5,-3.2) {CPO (Trust Region)};
\node[method] (m3) at (-4.5,-3.8) {Lyapunov-based};
\node[method] (m4) at (-4.5,-4.4) {PCPO (Projection)};
\draw[arr, gray!40] (lagr.south) -- (m1.north);
\draw[arr, gray!40] (m1.south) -- (m2.north);
\draw[arr, gray!40] (m2.south) -- (m3.north);
\draw[arr, gray!40] (m3.south) -- (m4.north);

\node[method] (s1) at (-1.5,-2.6) {Safety Layer (QP)};
\node[method] (s2) at (-1.5,-3.2) {LTL Shielding};
\node[method] (s3) at (-1.5,-3.8) {Human Oversight};
\draw[arr, gray!40] (shield.south) -- (s1.north);
\draw[arr, gray!40] (s1.south) -- (s2.north);
\draw[arr, gray!40] (s2.south) -- (s3.north);

\node[method] (r1) at (1.5,-2.6) {CVaR Optimization};
\node[method] (r2) at (1.5,-3.2) {Distributional RL};
\draw[arr, gray!40] (risk.south) -- (r1.north);
\draw[arr, gray!40] (r1.south) -- (r2.north);

\node[method] (ma1) at (4.5,-2.6) {MACPO (Centralized)};
\node[method] (ma2) at (4.5,-3.2) {Decentralized ($\kappa$-hop)};
\node[method] (ma3) at (4.5,-3.8) {Shielded MARL};
\node[method] (ma4) at (4.5,-4.4) {Stackelberg SafeRL};
\draw[arr, gray!40] (marl.south) -- (ma1.north);
\draw[arr, gray!40] (ma1.south) -- (ma2.north);
\draw[arr, gray!40] (ma2.south) -- (ma3.north);
\draw[arr, gray!40] (ma3.south) -- (ma4.north);

\node[font=\scriptsize\itshape, text=orange!60!black] at (-4.5,-4.85) {Sec.~4.1};
\node[font=\scriptsize\itshape, text=teal!60!black]   at (-1.5,-4.25) {Sec.~4.2};
\node[font=\scriptsize\itshape, text=purple!60!black]  at (1.5,-3.65)  {Sec.~4.3};
\node[font=\scriptsize\itshape, text=cyan!60!black]    at (4.5,-4.85)  {Sec.~4.4};

\end{tikzpicture}
\caption{Taxonomy of SafeRL and SafeMARL methods surveyed in this paper. Methods are categorized into four groups: constrained optimization approaches (Sec.~4.1), safety shield mechanisms (Sec.~4.2), risk-sensitive methods (Sec.~4.3), and multi-agent extensions (Sec.~4.4).}
\label{fig:methods_taxonomy}
\end{figure}

\subsection{Lagrangian-based Policy Optimization} One broad class of SafeRL algorithms uses the primal-dual (Lagrangian) approach discussed earlier to enforce constraints. The idea is to transform the constrained problem into a sequence of unconstrained problems with adjusted rewards.

\textbf{\bf Lagrangian Actor-Critic:} In this approach, one augments the standard RL loss with penalty terms for constraint costs. For example, one can define a penalized reward $r_\lambda(s,a) = r(s,a) - \lambda c(s,a)$ (eqn (11) in \cite{Tessler2019}) for a single-constraint problem, where $\lambda$ is treated as a learnable parameter. An actor-critic algorithm (Algorithm 1 in \cite{Tessler2019}) can then be used:
The \emph{actor} (policy $\pi_\theta$) is updated with respect to the penalized objective $J_{\text{pen}}(\pi) = J(\pi) - \lambda J_c(\pi)$, using policy gradient or other optimization.
The \emph{critic(s)} estimate both the value of the reward and the cost (often one critic for $V^\pi(s)$ and one for $V_c^\pi(s)$).
The Lagrange multiplier $\lambda$ is updated by gradient ascent on the constraint satisfaction term, e.g. $\lambda \leftarrow \lambda + \beta (J_c(\pi) - d)$.
This simple scheme is often called the \textit{Lagrange method} or \textit{reward shaping method} in safe RL. It was used in early safe deep RL implementations (e.g., \cite{Tessler2019} for safe DQN with constraints, and policy-gradient variants, i.e., Trust Region Policy Optimization and Proximal Policy Optimization in \cite{Ray2019}). While straightforward, a drawback is that the penalty coefficient $\lambda$ can be hard to tune (``Our baseline results for constrained RL indicate a need for
stronger and/or better-tuned algorithms to succeed on Safety Gym environments'' as quoted in \cite{Ray2019}) and the method does not guarantee strict constraint satisfaction until convergence. The agent might violate constraints during learning if $\lambda$ is not large enough, or conversely, learn too slowly if $\lambda$ is too large initially.

\begin{figure}[t!]
\centering
\begin{tikzpicture}[
  every node/.style={font=\small},
  block/.style={rectangle, draw=#1!70!black, fill=#1!8, rounded corners=3pt,
    minimum height=1.0cm, minimum width=2.8cm, align=center, line width=0.8pt},
  block/.default=blue,
  arr/.style={-{Stealth[length=2.5mm]}, line width=0.7pt, #1!70!black},
  arr/.default=gray,
  lbl/.style={font=\scriptsize, fill=white, inner sep=1pt},
]

\node[block=teal, minimum width=3.0cm] (env) at (0, 3.5) {\textbf{Environment}};

\node[block=blue] (actor) at (0, 1.5) {\textbf{Actor}\\$\pi_\theta(a|s)$};
\node[block=orange, minimum width=2.4cm] (lam) at (5.0, 1.5) {\textbf{Multiplier}\\$\lambda$};

\node[block=green] (vr) at (-3.0, -0.8) {\textbf{Reward Critic}\\$\hat{V}^\pi(s)$};
\node[block=red] (vc) at (3.0, -0.8) {\textbf{Cost Critic}\\$\hat{V}_c^\pi(s)$};


\draw[arr=teal] ([xshift=-0.3cm]env.south) -- node[lbl, left] {$s_t$} ([xshift=-0.3cm]actor.north);

\draw[arr=blue] ([xshift=0.3cm]actor.north) -- node[lbl, right] {$a_t$} ([xshift=0.3cm]env.south);

\draw[arr=green] ([xshift=-0.8cm]env.south) -- node[lbl, left] {$r_t$} (vr.north);

\draw[arr=red] ([xshift=0.8cm]env.south) -- node[lbl, right] {$c_t$} (vc.north);

\draw[arr=green] (vr.north east) -- node[lbl, above, sloped] {$\hat{A}^\pi(s,a)$} ([xshift=-0.4cm]actor.south);

\draw[arr=red] (vc.north west) -- node[lbl, above, sloped] {$\hat{A}_c^\pi(s,a)$} ([xshift=0.4cm]actor.south);

\draw[arr=orange] (vc.east) -- node[lbl, below, sloped, font=\scriptsize] {$J_c(\pi) - d$} (lam.south);

\draw[arr=orange, dashed] (lam.west) -- node[lbl, above] {$\lambda$ penalty} (actor.east);

\node[font=\scriptsize, align=left, text=gray!70!black] at (1.0, -2.3) {
  Policy update: $\theta \leftarrow \theta + \alpha \nabla_\theta [J(\pi) - \lambda J_c(\pi)]$
  \qquad Dual update: $\lambda \leftarrow [\lambda + \beta(J_c(\pi) - d)]_+$
};

\end{tikzpicture}
\caption{Lagrangian actor-critic architecture for SafeRL. The actor (policy network) is updated using advantage estimates from both a reward critic and a cost critic. The Lagrange multiplier $\lambda$ is adapted online based on constraint violation, automatically balancing reward maximization with safety.}
\label{fig:lagrangian_ac}
\end{figure}

\textbf{\bf Projected Lagrangian (Constrained Policy Optimization):} Achiam \emph{et al.} \cite{Achiam2017} introduced \textbf{Constrained Policy Optimization (CPO)}, a landmark algorithm that improves upon the basic Lagrangian method by ensuring each policy update is safe. CPO is built on trust-region policy optimization:
At each iteration, it solves a local constrained optimization: maximize policy improvement subject to a constraint that the cost does not increase beyond a small tolerance. This is done by a quadratic approximation of the objective and a linear approximation of the constraints (using policy gradient and cost gradient), then solving a convex subproblem.
If the proposed update violates the constraint (predicted cost increase too high), CPO backtracks or projects the policy update to the nearest feasible update.
CPO provides theoretical guarantees of \emph{near-constraint satisfaction at each iteration}: essentially, it never overshoots the constraint by more than a certain second-order error term, keeping training safe.
CPO demonstrated that one can train neural network policies for control tasks while maintaining safety throughout training \cite{Achiam2017}. It was the first general-purpose safe RL algorithm with such guarantees. However, CPO is more complex and computationally heavier than standard policy gradient (due to solving the constrained optimization subproblem each step). It also requires a reliable estimation of the cost value and cost advantage, which can be challenging; for complex environments, cost estimates may be derived from a separate classifier trained to identify safe/unsafe actions \cite{chirra2025safetyfeedbackconstrainedrl}.

\begin{proposition}[{\bf CPO Trust Region Safety Guarantee \cite{Achiam2017}}]
Let $\pi_k$ be the current feasible policy and $\pi_{k+1}$ be the new policy obtained by solving the CPO trust region optimization problem:
\begin{align}
\pi_{k+1}
& =
\arg \max_{\pi \in \Pi_{\theta}}
\; \mathbb{E}_{s \sim d^{\pi_k},\, a \sim \pi}
\!\left[A^{\pi_k}(s,a)\right] \nonumber \\
& \text{s.t.} \quad
J_{C_i}(\pi_k)
+ \frac{1}{1-\gamma}
\mathbb{E}_{s \sim d^{\pi_k},\, a \sim \pi}
\!\left[A_{C_i}^{\pi_k}(s,a)\right]
\le d_i,~\forall i, \nonumber \\
& \qquad \overline{D}_{\mathrm{KL}}(\pi \, \| \, \pi_k) \le \delta.
\label{eq:cpo_opt}
\end{align}

The new policy $\pi_{k+1}$ is guaranteed to satisfy the original cost constraint $J_{C_i}$ up to a bounded error term:
\begin{equation}
J_{C_i}(\pi_{k+1})
\le
d_i +
\frac{\sqrt{2\delta}\,\gamma\,\epsilon_{C_i}^{\pi_{k+1}}}{(1-\gamma)^2}.
\label{eq:cpo_bound}
\end{equation}

Here, $\epsilon_{C_i}^{\pi_{k+1}} =
\max_s \!\big|\mathbb{E}_{a \sim \pi_{k+1}}\![A_{C_i}^{\pi_k}(s,a)]\big|$.
\end{proposition}

\begin{proof}[Sketch of proof, detail in \cite{Achiam2017}]
\textbf{(1) Start from the cost performance bound.}
For any cost $C_i$ and policies $\pi',\pi$, \cite{Achiam2017} give (Corollary~2):
\begin{align*}
J_{C_i}(\pi') - J_{C_i}(\pi)
&\;\le\;
\frac{1}{1-\gamma}\,
\mathbb{E}_{s\sim d^\pi,\,a\sim\pi'}\!\big[A^{\pi}_{C_i}(s,a)\big] \\
&\;+\;
\frac{2\gamma}{(1-\gamma)^2}\,\epsilon^{\pi'}_{C_i}\,
\mathbb{E}_{s\sim d^\pi}\!\big[D_{TV}(\pi'\|\pi)[s]\big],
\end{align*}
where $\epsilon^{\pi'}_{C_i}=\max_s \big|\mathbb{E}_{a\sim\pi'}[A^{\pi}_{C_i}(s,a)]\big|$.

\textbf{(2) Replace TV by average KL under a trust region.}
By Pinsker's inequality and Jensen's inequality,
\[
\mathbb{E}_{s\sim d^\pi}[D_{TV}(\pi'\|\pi)[s]] \;\le\;
\sqrt{\tfrac{1}{2}\,\mathbb{E}_{s\sim d^\pi}[D_{\mathrm{KL}}(\pi'\|\pi)[s]]}.
\]
If $\pi'$ is produced by the CPO update with trust region
$\overline D_{\mathrm{KL}}(\pi'\|\pi)\!=\!\mathbb{E}_{s\sim d^\pi}[D_{\mathrm{KL}}(\pi'\|\pi)[s]] \le \delta$,
then
$
\mathbb{E}_{s\sim d^\pi}[D_{TV}(\pi'\|\pi)[s]] \le \sqrt{\delta/2}.
$

\textbf{(3) Apply to the constrained subproblem and re-arrange.}
In the CPO subproblem, the surrogate constraint enforces
$
J_{C_i}(\pi_k) + \tfrac{1}{1-\gamma}\,
\mathbb{E}_{s\sim d^{\pi_k},a\sim\pi}\big[A^{\pi_k}_{C_i}(s,a)\big] \le d_i.
$
Plugging $\pi'=\pi_{k+1}$ and $\pi=\pi_k$ into the inequality of Step~(1), and
then substituting the TV$\to$KL bound from Step~(2) yields
\[
J_{C_i}(\pi_{k+1}) \le
d_i + \frac{2\gamma}{(1-\gamma)^2}\,\epsilon^{\pi_{k+1}}_{C_i}\,\sqrt{\frac{\delta}{2}}
\;=\;
d_i + \frac{\sqrt{2\delta}\,\gamma\,\epsilon^{\pi_{k+1}}_{C_i}}{(1-\gamma)^2},
\]
which is exactly the stated bound.
\end{proof}

\begin{figure}[t!]
\centering
\begin{tikzpicture}[
  every node/.style={font=\small},
]
\draw[-{Stealth[length=2mm]}, gray!30, line width=0.4pt] (-0.5,0) -- (7.5,0) node[right, font=\scriptsize, text=gray!50] {$\theta_1$};
\draw[-{Stealth[length=2mm]}, gray!30, line width=0.4pt] (0,-0.5) -- (0,5.8) node[above, font=\scriptsize, text=gray!50] {$\theta_2$};

\fill[green!8] (0,2.2) -- (7.5,0.6) -- (7.5,5.8) -- (0,5.8) -- cycle;

\draw[red!70!black, line width=1.2pt] (0,2.2) -- (7.5,0.6);

\node[font=\scriptsize\bfseries, text=red!70!black, rotate=-12.2, anchor=north] at (6.0,0.95) {Constraint: $J_c(\pi) = d$};

\node[font=\scriptsize\itshape, text=green!50!black] at (6.0,4.8) {Feasible: $J_c \leq d$};
\node[font=\scriptsize\itshape, text=red!50!black] at (5.5,0.3) {Infeasible: $J_c > d$};

\fill[blue!70!black] (3,3.5) circle (3pt);
\node[left, font=\small\bfseries, text=blue!70!black] at (2.85,3.5) {$\theta_k$};

\draw[blue!50!black, line width=0.8pt, dashed] (3,3.5) circle (1.5);

\node[font=\scriptsize, text=blue!50!black, align=center] at (0.8,5.4) {Trust region\\$D_{\mathrm{KL}} \leq \delta$};
\draw[-{Stealth[length=1.5mm]}, blue!40, line width=0.4pt] (1.3,5.0) -- (1.8,4.6);

\draw[-{Stealth[length=2mm]}, gray!50, line width=0.5pt, dotted] (4.5,5.3) -- (6.5,5.3)
  node[right, font=\scriptsize, text=gray!60!black] {$\nabla J(\pi)$};

\draw[-{Stealth[length=3mm]}, orange!80!black, line width=1.2pt]
  (3,3.5) -- (4.0,2.0);
\node[font=\scriptsize\bfseries, text=orange!80!black, anchor=north west] at (4.1,1.85) {Standard PG};
\node[font=\tiny, text=orange!70!black, anchor=north east] at (3.9,0.9) {(violates constraint)};

\draw[-{Stealth[length=3mm]}, green!60!black, line width=1.5pt]
  (3,3.5) -- (4.4,3.1);

\fill[green!60!black] (4.4,3.1) circle (3pt);

\node[font=\small\bfseries, text=green!60!black, anchor=west] at (4.6,3.35) {$\theta_{k+1}$};
\node[font=\scriptsize\bfseries, text=green!60!black, anchor=west] at (4.6,2.95) {CPO update};

\draw[-{Stealth[length=2.5mm]}, purple!70!black, line width=0.8pt, dashed]
  (4.0,2.0) -- (4.4,3.1);

\node[font=\scriptsize, text=purple!70!black, anchor=west] at (4.5,2.45) {PCPO projection};

\end{tikzpicture}
\caption{Geometric illustration of Constrained Policy Optimization (CPO). At each iteration, CPO seeks the policy update that maximizes reward improvement within a trust region (KL divergence constraint) while staying in the feasible set ($J_c \leq d$). If the standard policy gradient step would violate the constraint, CPO projects or backtracks the update to the feasible boundary. PCPO achieves this via explicit projection (dashed purple arrow).}
\label{fig:cpo_trust_region}
\end{figure}

Many subsequent works have built on or modified CPO:
\textbf{PCPO (Projection-based CPO):} an algorithm that explicitly projects the policy gradient to the feasible set defined by constraint gradients \cite{yang2020projection}. It is a simplification that avoids solving a quadratic program but still aims to keep updates safe by geometric projection.

\textbf{TRPO-Lagrangian:} A simpler baseline where one applies a trust-region update on the penalized objective $J - \lambda J_c$ instead of solving a constrained QP. This does not guarantee strict feasibility but often empirically manages constraint violations by proper $\lambda$ adaptation. OpenAI's Safety Gym benchmark release \cite{openai_safety,Ray2019} used such baselines\footnote{\url{https://github.com/openai/safety-starter-agents}}.

\textbf{Actor-Critic with Lyapunov:} Chow \emph{et al.} \cite{Chow2018} proposed using a Lyapunov function (a monotonic function of the cost-to-go) to derive a safe update rule. They ensure the new policy does not increase a Lyapunov function, which in turn guarantees the constraint remains satisfied. This can be seen as another form of trust-region or projection method specialized using Lyapunov theory.

\textbf{Off-policy and Model-based extensions:} While most policy optimization methods are on-policy, there have been adaptations to off-policy learning:
\emph{Safe DDPG or TD3:} by incorporating a cost critic and Lagrange multiplier, one can train deterministic policies (as in DDPG) with a constraint. For example, a constrained variant of TD3 (Twin Delayed Deep Deterministic Policy Gradient\cite{fujimoto2018addressing}\footnote{\url{https://spinningup.openai.com/en/latest/algorithms/td3.html}}) was proposed by \cite{zhang2023x,yang2025proactiveconstrainedpolicyoptimization}.

\textbf{Model-based SafeRL:} Berkenkamp \emph{et al.} \cite{Berkenkamp2017,Berkenkamp2015,berkenkamp2016safe} used Gaussian process models of the dynamics to ensure safety. They construct a stabilizing controller (via control theory) that acts as a baseline policy and only allow the learning agent to explore if it can certify (using a Lyapunov condition) that the new policy is safe. While not directly a CMDP approach, this provides an alternative angle: blending traditional control safety with RL exploration.

\subsection{Safety Shields and Action Correction} Another category of SafeRL methods focuses on safe exploration: how to prevent an agent from ever taking an action that could lead to catastrophe. These methods act as a layer on top of any standard RL algorithm, modifying or filtering its actions (see Figure~\ref{fig:safety_shield}):

\begin{figure}[t!]
\centering
\begin{tikzpicture}[
  every node/.style={font=\small},
  block/.style={rectangle, draw=#1!70!black, fill=#1!8, rounded corners=3pt,
    minimum height=1.0cm, minimum width=2.2cm, align=center, line width=0.8pt},
  block/.default=blue,
  decision/.style={diamond, draw=red!70!black, fill=red!5, minimum width=1.6cm,
    minimum height=1.2cm, align=center, line width=0.7pt, font=\scriptsize,
    inner sep=1pt},
  arr/.style={-{Stealth[length=2.5mm]}, line width=0.8pt, #1!70!black},
  arr/.default=gray,
  lbl/.style={font=\scriptsize, fill=white, inner sep=1pt},
]

\node[block=blue, minimum width=2.4cm] (agent) {\textbf{RL Agent}\\$\pi_\theta(a|s)$};

\node[right=1.0cm of agent] (act_label) {};

\node[block=red, right=2.8cm of agent, minimum width=2.6cm, minimum height=1.3cm] (shield) {\textbf{Safety Shield}\\(Filter / QP / \\Formal Verifier)};

\node[block=teal, right=2.8cm of shield, minimum width=2.4cm] (env) {\textbf{Environment}};

\node[block=orange, below=1.2cm of shield, minimum width=2.6cm] (model) {\textbf{Safety Model}\\(Dynamics / LTL spec /\\Learned model)};

\draw[arr=blue] (agent.east) -- node[lbl, above] {Proposed $a_t$} (shield.west);

\draw[arr=green] (shield.east) -- node[lbl, above] {Safe $a'_t$} (env.west);

\draw[arr=teal] (env.south) -- ++(0,-3.5) -| node[lbl, pos=0.25, below] {$s_{t+1}, r_t, c_t$} (agent.south);

\draw[arr=orange] (model.north) -- node[lbl, right, font=\scriptsize, align=center] {Safe/unsafe\\prediction} (shield.south);

\draw[arr=teal, dashed] (env.south west) -- node[lbl, above, sloped] {$s_t$} (model.north east);

\node[right=0.1cm of shield.north east, font=\scriptsize\itshape, text=green!50!black] {If safe: $a'_t = a_t$};
\node[right=0.1cm of shield.south east, font=\scriptsize\itshape, text=red!50!black] {If unsafe: $a'_t = a_{\text{safe}}$};

\end{tikzpicture}
\caption{Safety shield / action correction pipeline. The RL agent proposes an action $a_t$, which passes through a safety shield before reaching the environment. The shield uses a safety model (dynamics model, formal specification, or learned predictor) to check whether $a_t$ is safe. If unsafe, it corrects the action to the nearest safe alternative $a'_t$, guaranteeing no constraint violation.}
\label{fig:safety_shield}
\end{figure}

\begin{itemize}
    \item \textbf{Safety Shield / Filter:} A mechanism that monitors the agent’s chosen action and overrides it if it is deemed unsafe. The override might be a safe default action or the closest safe action. Dalal \emph{et al.} \cite{Dalal2018} introduced a \emph{safety layer} that solves a quadratic program (QP) in continuous action spaces to minimally adjust the action such that predicted next state stays within safety bounds. This method guaranteed zero constraint violations during training on those tasks. However, it requires a model (or learned model) to predict constraint violations.

    \item \textbf{Shielding via formal methods:} Alshiekh \emph{et al.} \cite{alshiekh2018safe} and later ElSayed-Aly \emph{et al.} \cite{ElSayed2021} (extended to multiagents) use formal verification and temporal logic to build shields. The idea is to pre-compute a set of forbidden state-action pairs using model checking of an abstract model, or to synthesize a runtime observer from a formal specification. The shield then blocks any action that would lead into a bad state (violating the LTL safety specification) in finite steps. In multi-agent settings, as \cite{ElSayed2021} shows, one can have a centralized shield watching over joint actions or distributed shields for each agent.

    \item \textbf{Human or Oracle intervention:} In practical scenarios, one may employ a human overseer or a safety oracle to intervene when the agent is about to do something unsafe. While not a scalable solution for all time, during training it can prevent disasters. Safe RL with human intervention was studied in \cite{Saunders2018} where a human can cancel dangerous actions, and the agent is penalized for those. Over time the agent learns to avoid actions that would have been blocked.
\end{itemize}
    
Shielding approaches have the advantage of hard safety (no violations in theory), but they often rely on having additional knowledge: either a dynamics model, a predefined safe set, or an external supervisor. They also may introduce performance bias (the agent might become too conservative if the shield is not carefully designed, since it never experiences certain parts of state space). Combining shielding with CMDP-based learning is an interesting direction: one can use shielding in early training and gradually lift it as the agent’s own policy becomes safe with learned constraints.

\subsection{Risk-Sensitive and Distributional Methods} Although our focus is on constraint-based SafeRL, a brief mention of risk-sensitive RL is warranted as an alternative approach:
In risk-sensitive RL, instead of constraints, the optimization criterion itself is altered to account for risk. For example, one might maximize $U^{-1}( E[ U(\sum r) ] )$ where $U$ is a concave utility (exponential utility gives risk-aversion) or maximize $\text{CVaR}_{\alpha}(\sum r)$ of return at some confidence level $\alpha$.
Tamar \emph{et al.} \cite{Tamar2015} and others have developed policy gradient methods for CVaR (see \cite{Buerle2023MarkovDP} for an overview of risk-sensitive criteria in MDPs). These effectively try to ensure with high probability the return is above some level, which is conceptually similar to constraints on probabilities of bad events.
\textbf{Distributional RL} (as popularized by \cite{Bellemare2017}) learns the full distribution of returns. One can combine distributional RL with safety by focusing on the lower tail of the return distribution to ensure it is above some threshold. This is another way to encode safety without explicit constraints.
Risk-sensitive methods can sometimes be converted into CMDP style constraints. For instance, requiring CVaR$(\text{cost}) \le d$ is a constraint on a specific risk measure of cost. Solving such constraints often introduces auxiliary variables or uses sample-based approximations. While we do not detail these methods here, they are part of the broader SafeRL toolbox.

\subsection{Safe Multi-Agent Reinforcement Learning (SafeMARL)} SafeMARL extends the ideas above to multi-agent systems \cite{marl-book,weiss1999multiagent,wooldridge2009introduction,shoham2008multiagent}. We consider environments with $N$ agents, indexed by $i\in{1,\dots,N}$. A convenient formal model is a \textbf{constrained Markov game}, defined by $(\mathcal{S}, {\mathcal{A}_i}, P, {r_i}, {c_i^{(j)}}, \gamma)$. Here each agent $i$ chooses an action $a_i \in \mathcal{A}_i$, forming a joint action $\mathbf{a}=(a_1,\dots,a_N)$ that causes state transitions via $P(s'|s,\mathbf{a})$. Each agent can receive an individual reward $r_i(s,\mathbf{a})$ and has possibly its own set of cost functions $c_i^{(j)}(s,\mathbf{a})$ for $j=1\dots m_i$.

SafeMARL scenarios can be cooperative, competitive, or mixed
\begin{itemize}
    \item In fully \textbf{cooperative SafeMARL}, all agents share a common reward (or their rewards are aligned) and typically the safety constraints are also shared or at least all agents are interested in satisfying all constraints. For example, a team of robots might have a joint goal (maximize sum of rewards) and constraints like “no collisions among any robots” which is a global safety constraint.
    \item In \textbf{competitive or general-sum SafeMARL}, each agent has its own reward to maximize, and constraints might be individual (each agent has its own safety requirement) or shared (environment-level safety that everyone needs to uphold, like traffic rules). The solution concept might be a safe equilibrium \cite{Ganzfried2022} (e.g., a Nash equilibrium that respects constraints, related works \cite{AltmanEtAl2000_ConstrainedMarkovGames,Mccracken2004SafeSF} consider constrained Markov games and safe strategies for players) rather than a single policy optimization.
Most existing work in SafeMARL addresses cooperative settings, since even standard MARL is most tractable in either fully cooperative (centralized training for a team) or fully competitive (two-player zero-sum) cases. 
\end{itemize}
We highlight a few key approaches (see Figure~\ref{fig:safemarl_arch} for an architectural comparison):

\begin{figure}[t!]
\centering
\resizebox{\textwidth}{!}{%
\begin{tikzpicture}[
  every node/.style={font=\small},
  agent/.style={circle, draw=blue!70!black, fill=blue!10, minimum size=0.9cm,
    line width=0.7pt, font=\scriptsize\bfseries},
  ctrl/.style={rectangle, draw=purple!70!black, fill=purple!10, rounded corners=3pt,
    minimum height=0.9cm, minimum width=2.8cm, align=center, line width=0.8pt},
  envbox/.style={rectangle, draw=teal!70!black, fill=teal!8, rounded corners=3pt,
    minimum height=0.8cm, minimum width=2.4cm, align=center, line width=0.7pt},
  cbox/.style={rectangle, draw=red!60!black, fill=red!8, rounded corners=2pt,
    minimum height=0.6cm, align=center, line width=0.6pt, font=\scriptsize},
  arr/.style={-{Stealth[length=2mm]}, line width=0.6pt, gray!60!black},
  lbl/.style={font=\scriptsize, fill=white, inner sep=1pt},
  title/.style={font=\small\bfseries, text=gray!70!black},
]

\node[title] at (-3.5,3.5) {(a) Centralized SafeMARL};

\node[ctrl] at (-3.5,2.2) (cc) {\textbf{Central Controller}\\Joint policy $\pi(\mathbf{a}|s)$};

\node[cbox, minimum width=3.0cm] at (-3.5,0.8) (gc) {Global constraint: $J_{c}^{\text{global}}(\pi) \leq d$};

\node[agent] at (-5.0,-0.8) (a1) {$i{=}1$};
\node[agent] at (-3.5,-0.8) (a2) {$i{=}2$};
\node[agent] at (-2.0,-0.8) (a3) {$i{=}N$};
\node[font=\scriptsize] at (-2.75,-0.8) {$\cdots$};

\node[envbox] at (-3.5,-2.2) (env1) {\textbf{Shared Environment}};

\draw[arr] (cc.south) -- (gc.north);
\draw[arr] (gc.south) -- ++(0,-0.4) -| (a1.north);
\draw[arr] (gc.south) -- (a2.north);
\draw[arr] (gc.south) -- ++(0,-0.3) -| (a3.north);

\draw[arr, blue!50] (a1.south) -- ([xshift=-0.6cm]env1.north);
\draw[arr, blue!50] (a2.south) -- (env1.north);
\draw[arr, blue!50] (a3.south) -- ([xshift=0.6cm]env1.north);

\draw[arr, teal!70!black] (env1.west) -- ++(-1.3,0) |- node[lbl, pos=0.25, left] {\scriptsize Global $s$} (cc.west);

\node[title] at (4.0,3.5) {(b) Decentralized SafeMARL};

\node[agent, fill=blue!15] at (2.2,2.2) (d1) {$i{=}1$};
\node[agent, fill=blue!15] at (4.0,2.2) (d2) {$i{=}2$};
\node[agent, fill=blue!15] at (5.8,2.2) (d3) {$i{=}N$};
\node[font=\scriptsize] at (4.9,2.2) {$\cdots$};


\node[cbox] at (2.2,0.8) (lc1) {Local $J_{c_1} \leq d_1$};
\node[cbox] at (5.8,0.8) (lc3) {Local $J_{c_N} \leq d_N$};

\draw[{Stealth[length=1.5mm]}-{Stealth[length=1.5mm]}, orange!70!black, line width=0.6pt, dashed]
  (d1.east) -- node[above, font=\tiny, text=orange!60!black] {comm} (d2.west);
\draw[{Stealth[length=1.5mm]}-{Stealth[length=1.5mm]}, orange!70!black, line width=0.6pt, dashed]
  (d2.east) -- ++(0.4,0);

\draw[arr, red!50] (d1.south) -- node[lbl, right, font=\scriptsize, text=blue!60!black] {$\pi_1$} (lc1.north);
\draw[arr, red!50] (d3.south) -- node[lbl, right, font=\scriptsize, text=blue!60!black] {$\pi_N$} (lc3.north);

\node[envbox] at (4.0,-1.5) (env2) {\textbf{Shared Environment}};

\draw[arr, blue!50] (lc1.south) -- ++(0,-0.55) -| ([xshift=-0.5cm]env2.north);
\draw[arr, blue!50] (lc3.south) -- ++(0,-0.55) -| ([xshift=0.5cm]env2.north);

\draw[arr, teal!70!black] (env2.west) -- ++(-2.2,0) |- node[lbl, pos=0.25, left] {\scriptsize $o_1$} (d1.west);
\draw[arr, teal!70!black] (env2.east) -- ++(2.2,0) |- node[lbl, pos=0.25, right] {\scriptsize $o_N$} (d3.east);

\node[font=\scriptsize\itshape, text=orange!60!black] at (4.0,3.0) {$\kappa$-hop neighborhood};

\end{tikzpicture}%
}
\caption{Comparison of centralized vs.\ decentralized SafeMARL architectures. (a)~Centralized: a central controller observes global state and optimizes a joint policy subject to a global safety constraint (e.g., MACPO). (b)~Decentralized: each agent $i$ has a local policy $\pi_i$, local observations $o_i$, and local constraint approximations. Agents coordinate via limited communication within a $\kappa$-hop neighborhood.}
\label{fig:safemarl_arch}
\end{figure}

\begin{proposition}[{\bf Safe MARL Monotonic Improvement Guarantee~\cite{Gu2023MultiRobot}}]
In the cooperative setting, a team of $n$ agents aims to maximize a joint reward $J(\pi)$ while ensuring that each agent $i$ satisfies its own set of cost constraints
$
J_{C_j^i}(\pi) \le c_j^i,\ \forall i\in\{1,\ldots,n\},\ j\in\mathcal{C}_i.
$
The paper introduces a \emph{Safe Multi-Agent Policy Iteration} procedure that provides the following guarantees at every iteration $k$:
\begin{enumerate}
\item $J(\pi_{k+1}) \ge J(\pi_k)$ \quad (monotonic reward improvement),
\item $J_{C_j^i}(\pi_k) \le c_j^i$ for all $i,j$ \quad (per-iteration constraint satisfaction).
\end{enumerate}
These guarantees are achieved through a sequential update scheme in which each agent $i_h$ solves a constrained optimization problem with agent-wise KL trust region radii chosen to ensure that every other agent's constraint remains bounded by its threshold.
\end{proposition}

\begin{proof}[Sketch of proof]
\textbf{Step 1: Multi-agent surrogate decomposition.}
Let $A_\pi$ be the joint advantage under the current joint policy $\pi_k$.
By the multi-agent advantage decomposition (Lemma~1 in \cite{Gu2023MultiRobot}), for any ordering $i_1{:}n$,
\begin{align*}
&\mathbb{E}_{s\sim\rho^{\pi_k},\, a\sim \pi_{k+1}}\big[A_{\pi_k}(s,a)\big] \\
&= \sum_{h=1}^n
\mathbb{E}_{s\sim\rho^{\pi_k},
\, a_{i_{1{:}h-1}}\sim \pi_{i_{1{:}h-1}}^{k+1},
\, a_{i_h}\sim \pi_{i_h}^{k+1}}
\Big[
A^{\,i_h}_{\pi_k}\!\big(s, a_{i_{1{:}h-1}}, a_{i_h}\big)
\Big].
\end{align*}
Define the one-agent surrogate
\[
L^{\,i_{1{:}h}}_{\pi_k}(\pi_{i_{1{:}h-1}}^{k+1},\pi_{i_h})
\!=\!
\mathbb{E}_{s,a}\!\big[
A^{\,i_h}_{\pi_k}(s,a_{i_{1{:}h-1}},a_{i_h})
\big].
\]
A standard TRPO bound (applied in the multi-agent setting) gives
\begin{align*}
J(\pi_{k+1})
\,\ge\, J(\pi_k)\;+\;\sum_{h=1}^n
&\Big\{L^{\,i_{1{:}h}}_{\pi_k}(\pi_{i_{1{:}h-1}}^{k+1},\pi_{i_h}^{k+1}) \\
&\;-\;\nu\, D^{\max}_{\mathrm{KL}}\!\big(\pi_{i_h}^k,\pi_{i_h}^{k+1}\big)
\Big\},
\end{align*}
for $\nu=\tfrac{4\gamma}{(1-\gamma)^2}\max_{s,a}|A_{\pi_k}(s,a)|$.

\textbf{Step 2: Sequential argmax and the identity $L=0$ at the old policy.}
By construction, $L^{\,i_{1{:}h}}_{\pi_k}(\pi_{i_{1{:}h-1}}^{k+1},\pi_{i_h}^{k})=0$.
Each agent update $\pi_{i_h}^{k+1}$ is chosen to \emph{maximize} the penalized surrogate
$L^{\,i_{1{:}h}}_{\pi_k}(\pi_{i_{1{:}h-1}}^{k+1},\cdot)-\nu D^{\max}_{\mathrm{KL}}(\pi_{i_h}^k,\cdot)$,
so replacing $\pi_{i_h}^{k+1}$ by $\pi_{i_h}^k$ yields a lower value. Summing over $h$ gives
$J(\pi_{k+1}) \ge J(\pi_k)$, proving \emph{monotonic improvement}.

\textbf{Step 3: Per-iteration feasibility via cost surrogates + KL budgets.}
For each agent $i$ and cost index $j$, Lemma~2 in \cite{Gu2023MultiRobot} bounds the change of cost under a joint update:
\[
J^i_{C_j}(\pi_{k+1}) \;\le\; J^i_{C_j}(\pi_k)
\;+\;L^{\,i}_{C_j,\pi_k}(\pi_i^{k+1})
\;+\;\nu^i_j\,\sum_{\ell=1}^n D^{\max}_{\mathrm{KL}}(\pi_\ell^k,\pi_\ell^{k+1}),
\]
with $\nu^i_j=\tfrac{4\gamma}{(1-\gamma)^2}\max_{s,a_i}|A^{\,i}_{C_j,\pi_k}(s,a_i)|$.
Choosing agent-wise KL radii $\delta_{i_h}$ ensures that, starting from a feasible $\pi_k$,
the sequential update of $i_h$ keeps every other agent's constraint bounded by its threshold; hence
$J^i_{C_j}(\pi_{k+1})\le c_j^i$ for all $i,j$. This proves \emph{per-iteration constraint satisfaction}.

Combining Steps 1--3 yields the proposition.
\end{proof}

\begin{itemize}
    \item \textbf{Centralized Training with Global Constraints:} A straightforward extension of single-agent SafeRL to multi-agent cooperative tasks is to treat the entire multi-agent system as one big agent with a joint action \cite{marl-book}. One can then apply CMDP methods on the joint system. For example, one can define a joint policy $\pi(\mathbf{a}|s)$ and a global cost $c_{\text{global}}(s,\mathbf{a})$ that encodes any violation by any agent. Then apply CPO or Lagrange methods on this joint policy. This was essentially the approach in the MACPO algorithm\footnote{\url{https://github.com/chauncygu/Multi-Agent-Constrained-Policy-Optimisation}} by \cite{Gu2023MultiRobot}: they derived a multi-agent version of the CPO update (ensuring monotonic improvement in team reward and satisfaction of safety constraints). In practice, they implemented MACPO with two variants: one using a centralized critic (accessible during training) that estimates global reward and cost, and another using a factorized approach (MAPPO-Lagrangian, See Lemma 1: Multiagent advantage decompositon in \cite{Gu2023MultiRobot} ) which is simpler and uses decentralized advantage estimates with a Lagrange penalty for costs. The challenge with centralized approaches is the scalability: the joint action space grows exponentially with number of agents, and a centralized policy might be impractical for many agents. It also requires a central controller during training (and possibly execution) that knows all agent's states, which might not be available in all applications.

    \item \textbf{Decentralized Safe Learning with Coordination:} An important research direction is how to achieve safe multi-agent learning without relying on a central entity or a global state accessible to all. Recent work by Zhang \emph{et al.} \cite{Zhang2024Scalable} introduced a scalable constrained policy optimization where each agent optimizes a localized objective that approximates the global safety. They use the concept of $\kappa$-hop neighborhood (each agent coordinates with others within $\kappa$ hops in a communication graph) to truncate the dependence on far-away agents. They proved that if each agent optimizes a local policy with these truncated safety constraints and updates sequentially, the overall system still improves reward and satisfies constraints. The resulting algorithm (Scalable MAPPO-Lagrangian) shows promising results on large multi-agent environments, demonstrating that strict centralization is not always necessary for SafeMARL. Another method for decentralization is to factor the safety constraints: \cite{ElSayed2021} did this via shields for subsets of agents. In general, one can attempt to decompose a global constraint into local constraints for each agent. For example, a global cost $c_{\text{global}}(s,\mathbf{a})$ might be decomposed as $c_{\text{global}} = \sum_i c_i(s,a_i)$ if the unsafe events are localized per agent. Then each agent could constrain its own $c_i$. However, not all safety constraints are additively separable; many (like collision avoidance) are inherently about joint configurations. This remains a hard problem: designing local reward/cost structures whose alignment with global safety yields provable guarantees.

    \item \textbf{Multi-agent Credit Assignment for Safety:} In multi-agent RL, credit assignment (determining each agent's contribution to global reward) is crucial. Similarly, for safety, one might need to assign “blame” or responsibility to individual agents for a safety violation. Approaches like difference rewards \cite{wolpert2002optimal,tumer2004survey} or shaped team rewards \cite{devlin2011potential} can be used to ensure each agent gets feedback about how its actions affected the global outcome. For SafeRL, one could design each agent’s cost signal such that it corresponds to the marginal increase in global risk due to that agent. Some initial works have considered approaches where each agent considering the safety constraints of others \cite{Gu2023MultiRobot}, though a general solution is open research (we outline this as a problem later).

    \item \textbf{Safe Equilibria and Non-Cooperative Agents:} For competitive settings, one could consider each agent solving its own CMDP subject to safety constraints, leading to a game where each agent's strategy must satisfy its own constraints. The concept of a \emph{constrained Nash equilibrium} arises: a profile of policies $\pi_1,\dots,\pi_N$ such that no agent can improve its reward without violating constraints given the other's policies. Algorithms to compute such equilibria are not well-developed; this might involve ideas from game theory (like best response dynamics with constraints or Lagrangian for each agent). One example in literature is safe multi-agent learning via Stackelberg games: one agent (leader) accounts for the follower's response. \cite{Zheng2024} apply a bilevel optimization (Stackelberg) to model an autonomous driving scenario with safety, effectively solving a two-agent safe RL where the vehicles plans with knowledge of the other’s constraints (for example, in road intersection environments). This is a rich area for future investigation.

    \item \textbf{Benchmarking SafeMARL:} The progress in SafeMARL has been accelerated by the introduction of benchmarks. Gu \emph{et al.} \cite{Gu2023MultiRobot} provided \emph{Safe Multi-Agent MuJoCo\footnote{\url{https://github.com/chauncygu/Safe-Multi-Agent-Mujoco}}}, \emph{Safe MARobosuite}, and \emph{Safe MA-IsaacGym}, which are multi-robot simulation tasks with safety constraints (like torque limits or collision constraints). These environments allow systematic evaluation of SafeMARL algorithms in settings requiring coordination. Similarly, for single-agent SafeRL, OpenAI’s \emph{Safety Gym\footnote{\url{https://github.com/openai/safety-gym}}} \cite{Ray2019} introduced a suite of continuous control tasks with hazards and constraints, which has become a standard testbed. More recently, \emph{Safety Gymnasium} \cite{ji2023safety} provides a modernized successor with Gymnasium API support, and algorithm libraries such as \emph{OmniSafe} \cite{ji2023omnisafeinfrastructureacceleratingsafe} and \emph{Safe Policy Optimization (SafePO)} \cite{ji2023safety} offer unified implementations of constrained RL baselines. For safe control benchmarking, \emph{Safe-Control-Gym} \cite{Yuan2022} integrates constraint-aware control tasks for robotics. Multi-agent driving environments such as \emph{SMARTS} \cite{zhou2020smarts} and single-agent driving suites like \emph{Highway-env} \cite{highway-env} further expand the available testbeds for safe RL research.

\end{itemize}
In summary, state-of-the-art SafeRL methods range from modified policy gradient algorithms (with theoretical guarantees like CPO) to pragmatic penalty-based methods, model-based safe exploration, and safety layers, whereas SafeMARL is exploring centralized vs. decentralized learning, coordination mechanisms, and safe policy equilibrium concepts. Table \ref{tab:methods} provides a high-level summary of key selected algorithms in SafeRL and SafeMARL.
\begin{table}[t!] \centering \small
\rowcolors{2}{teal!8}{white}
\begin{tabular}{p{4cm} p{9cm}}
\hline
\rowcolor{teal!20}
\textbf{Method/Algorithm} & \textbf{Description and Key Features} \\
\hline

Lagrangian actor-critic \cite{Tessler2019} & Add constraint cost as penalty to reward; update $\lambda$ online. Simple but may violate constraints before convergence. \\

Constrained Policy Optimization (CPO) \cite{Achiam2017} & Trust-region policy updates with theoretical guarantee of near-constraint satisfaction each iteration. Uses second-order approximations to ensure safety. \\

Lyapunov-based Policy Optimization \cite{Chow2018} & Uses a Lyapunov function (cost critic) to constrain updates. Guarantees decrease in an upper bound of cost. \\

Reward Constrained DQN \cite{Tessler2019} & DQN with a reward penalty for constraint, ensuring discrete actions respect cost limit in expectation. \\

Safe DDPG/TD3 (Lagrangian) & Extends continuous control off-policy algorithms with cost critics and Lagrange multipliers for constraints. \\

Safe Model-Based RL \cite{Berkenkamp2017} & Uses model uncertainty estimates and stability analysis to allow only proven-safe explorations. Ensures no violations under certain dynamics assumptions. \\ 

Safety Layer (action shield) \cite{Dalal2018} & A differentiable layer that projects chosen actions to the nearest safe action by solving a QP. Guarantees zero immediate violations given local dynamics linearization. \\

Shielding (LTL) \cite{ElSayed2021} & Pre-compute shields from formal specifications; filter multi-agent joint actions to avoid unsafe outcomes. Achieves provable safety with respect to spec. \\

Multi-Agent CPO (MACPO) \cite{Gu2023MultiRobot} & Extension of CPO for multi-agent teams. Centralized training, uses a joint policy or coordinated update. Demonstrated on multi-robot tasks. \\

Scalable Decentralized Safe MARL \cite{Zhang2024Scalable} & Each agent optimizes a local surrogate constrained problem using truncated observation of others. Achieves near-centralized performance with better scalability. \\ 

Safe MARL via Bilevel (Stackelberg) \cite{Zheng2024} & Models one agent as leader, others as followers in a game with safety constraints. Solves via bilevel optimization to account for interactive safety. \\

Safe MARL with Shielding \cite{ElSayed2021} & Combines MARL with runtime shielding (central or factored) to ensure no unsafe joint actions are taken during learning. \\ 
\hline 
\end{tabular} 
\caption{Representative SafeRL (single-agent) and SafeMARL (multi-agent) methods.} \label{tab:methods} 
\end{table}
\section{Safe RL and Safe MARL Libraries}
In Tables~\ref{tab:safeRL_libs} and \ref{tab:safeRL_envs}, we show various available popular libraries and environments for further research in safe RL and safe MARL.

\paragraph{Taxonomy of Safe Reinforcement Learning Libraries and Environments.}
To consolidate the growing ecosystem of reproducible and standardized implementations in Safe Reinforcement Learning (SafeRL) and Safe Multi-Agent Reinforcement Learning (SafeMARL),
Tables~\ref{tab:safeRL_libs} and~\ref{tab:safeRL_envs} summarize the major algorithmic libraries and environment suites currently used by the research community.
Table~\ref{tab:safeRL_libs} lists the most widely adopted open-source frameworks that implement constrained or risk-aware reinforcement learning algorithms under the Constrained Markov Decision Process (CMDP) formulation.
These include high-quality benchmark suites such as \emph{OmniSafe}, \emph{Safe Policy Optimization (SafePO)}, \emph{Safety Starter Agents}, and \emph{Safe-Control-Gym}.
Each library is characterized by its category, algorithmic focus, GitHub repository, and the environments it supports.
The table also provides typical application domains ranging from robotics and navigation to multi-agent coordination and associated references, thereby serving as a guide to reproducible SafeRL experimentation and comparison across methods.

\paragraph{Benchmarking Environments and Datasets.}
Complementing the algorithmic libraries, Table~\ref{tab:safeRL_envs} categorizes the principal benchmark environments and datasets that provide structured safety signals for policy learning.
These include the \emph{Safety-Gymnasium} and the legacy \emph{OpenAI Safety Gym} for constrained navigation and manipulation,
as well as the \emph{AI Safety Gridworlds} and \emph{SafeLife} platforms that test specification robustness and side-effect avoidance in discrete domains.
More complex continuous and multi-agent settings are covered by simulation frameworks such as \emph{SMARTS} for autonomous driving and \emph{Highway-env} for risk-sensitive control.
Each environment entry specifies whether it supports single-agent or multi-agent training, the structure of safety constraints or cost signals, representative tasks, and intended application areas.
Together, these two tables form a comprehensive taxonomy of tools that underpin experimental research in safe and trustworthy reinforcement learning.

\begin{table}[t!]
\centering
\renewcommand{\arraystretch}{1.2}
\setlength{\tabcolsep}{3pt}
\rowcolors{2}{teal!8}{white}
\caption{Taxonomy of Safe Reinforcement Learning and Multi-Agent Safety Libraries.}
\label{tab:safeRL_libs}
\small
\begin{tabular}{p{2.2cm} p{2.5cm} p{4.0cm} p{4.0cm}}
\hline
\textbf{Name} & \textbf{Category} & \textbf{Algorithms / Capabilities} & \textbf{Application Areas} \\
\hline
\textbf{OmniSafe} \cite{ji2023omnisafeinfrastructureacceleratingsafe} & Algorithm Library (Safe RL) & PPO/TRPO-Lagrangian, CPO, NPG-Lag, SAC-Lag; experiment management & Robotics, safe navigation \\
\textbf{SafePO} \cite{ji2023safety} & Safe RL \& Safe MARL & Lagrangian \& CPO-style variants; MARL support; training pipelines & Multi-agent coordination, safe control \\
\textbf{Safety Starter Agents} \cite{Ray2019} & Baselines / Reference & CPO, PPO-Lag, TRPO-Lag, DDPG-Lag baselines & Benchmarking baselines \\
\textbf{Safe-Control-Gym} \cite{Yuan2022} & Benchmarking Suite & Constraint-aware control tasks; integrates SB3 and RL agents & Safe robotics \& continuous control \\
\hline
\end{tabular}
\end{table}

\begin{table}[t!]
\centering
\renewcommand{\arraystretch}{1.2}
\setlength{\tabcolsep}{3pt}
\rowcolors{2}{teal!8}{white}
\caption{Taxonomy of Environments and Datasets for Safe RL and Safe MARL.}
\label{tab:safeRL_envs}
\small
\begin{tabular}{p{2.2cm} p{1.2cm} p{3.5cm} p{3.0cm} p{2.5cm}}
\hline
\textbf{Name} & \textbf{Agent Type} & \textbf{Key Features / Signals} & \textbf{Example Tasks} & \textbf{Citations} \\
\hline
\textbf{Safety Gymnasium} & Both & Cost signals \& constraints (hazards); Gymnasium API & Goal/Button/Push with hazards & \cite{ji2023safety} \\
\textbf{OpenAI Safety Gym} & Single & Original cost/constraint tasks (archived) & Point/Car/Doggo: Goal/Button/Push & \cite{Ray2019} \\
\textbf{AI Safety Gridworlds} & Single & Side-effect \& reward specification tests & Multiple gridworlds for specification failures & \cite{leike2017aisafetygridworlds} \\
\textbf{SafeLife} & Single & Procedurally generated levels; side-effect metrics & Life-like tasks minimizing side effects & \cite{wainwright2021safelife10exploringeffects} \\
\textbf{SMARTS} & Multi & Traffic simulation, risk modeling for multi-agent driving & Merging, intersection, adversarial driving & \cite{zhou2020smarts} \\
\textbf{Highway-env} & Single & Collision-avoidance objectives; Gym-compatible & Highway, merge, roundabout, parking & \cite{highway-env} \\
\hline
\end{tabular}
\end{table}

\section{Open Research Challenges and Future Directions} While significant progress has been made in SafeRL and SafeMARL, many challenges remain open. In this section, we present five research problems that, if solved, would substantially advance the field (see Figure~\ref{fig:open_problems}). Three of these pertain specifically to SafeMARL, reflecting the newer nature of multi-agent safety. For each problem, we describe the motivation, outline possible approaches (steps toward a solution), and reference relevant prior work to build upon.

\begin{figure}[t!]
\centering
\begin{tikzpicture}[
  every node/.style={font=\small},
  problem/.style={rectangle, draw=#1!70!black, fill=#1!8, rounded corners=4pt,
    minimum height=1.1cm, minimum width=3.6cm, align=center, line width=0.8pt},
  problem/.default=blue,
  conn/.style={{Stealth[length=2mm]}-{Stealth[length=2mm]}, line width=0.5pt,
    gray!40, dashed},
  tag/.style={font=\tiny\bfseries, text=white, fill=#1!70!black, rounded corners=1.5pt,
    inner sep=2pt},
  tag/.default=blue,
]

\node[problem=orange] (p1) at (0,3.5) {\textbf{P1:} Zero-Violation\\Safe Exploration};
\node[problem=teal] (p2) at (-4.0,1.5) {\textbf{P2:} Safety Under\\Partial Observability};
\node[problem=blue] (p3) at (-3.0,-1.5) {\textbf{P3:} Decentralized\\SafeMARL};
\node[problem=purple] (p4) at (3.0,-1.5) {\textbf{P4:} Competitive\\SafeMARL};
\node[problem=cyan] (p5) at (4.0,1.5) {\textbf{P5:} Non-Stationary\\Multi-Agent Safety};

\node[tag=orange, above right=-0.1cm and -0.6cm of p1] {Single-Agent};
\node[tag=teal, above right=-0.1cm and -0.6cm of p2] {Single-Agent};
\node[tag=blue, above right=-0.1cm and -0.6cm of p3] {Multi-Agent};
\node[tag=purple, above right=-0.1cm and -0.6cm of p4] {Multi-Agent};
\node[tag=cyan, above right=-0.1cm and -0.6cm of p5] {Multi-Agent};

\draw[conn] (p1) -- node[font=\tiny, fill=white, inner sep=1pt, sloped, above] {safe exploration} (p2);
\draw[conn] (p1) -- node[font=\tiny, fill=white, inner sep=1pt, sloped, above] {guarantees} (p5);
\draw[conn] (p2) -- node[font=\tiny, fill=white, inner sep=1pt, sloped, above] {local obs.} (p3);
\draw[conn] (p3) -- node[font=\tiny, fill=white, inner sep=1pt, below] {no central authority} (p4);
\draw[conn] (p4) -- node[font=\tiny, fill=white, inner sep=1pt, sloped, above] {adaptation} (p5);
\draw[conn] (p3) -- node[font=\tiny, fill=white, inner sep=1pt, sloped, above] {coordination} (p5);

\node[font=\footnotesize\itshape, text=gray!60!black, align=center] at (0,0.8) {Problems are\\interconnected};

\end{tikzpicture}
\caption{Five open research problems in SafeRL and SafeMARL, and their interconnections. Problems~1--2 focus on single-agent challenges (zero-violation guarantees and partial observability), while Problems~3--5 address multi-agent settings (decentralized safety, competitive equilibria, and non-stationarity). Dashed lines indicate how progress on one problem can benefit others.}
\label{fig:open_problems}
\end{figure}

Most SafeRL algorithms guarantee safety in expectation or asymptotically, but they often allow some violations during learning (especially early on). In high-stakes applications, even a single catastrophic failure is unacceptable. The research challenge is to design RL methods that ensure \emph{zero (or provably bounded) constraint violations throughout the entire training process}, without relying on a human in the loop.

Ensuring no violations typically requires either very conservative exploration or prior knowledge (dynamics models, safe baseline policy). Too conservative an approach can severely slow down learning. Balancing caution with exploration is tricky, as overly restrictive safety can trap the policy in a local optima (not exploring better solutions).

The concept of never violating constraints relates to \textbf{safe exploration}. Moldovan and Abbeel (2012) and others studied conditions for “safe policy learning” where certain states are absorbing traps (unsafe) and should be avoided forever. Approach like \textit{learning with a safety critic} \textit{mentor-assisted exploration} \cite{Saunders2018,zhou2018safetyawareapprenticeshiplearning,Curi2020} have been tried. However, a general solution remains elusive, especially for high-dimensional continuous tasks. Solving this problem would likely require combining learning with elements of control theory or formal methods to get the needed guarantees.

Many real-world problems are partially observable (POMDPs) – the agent does not have full knowledge of the true state relevant to safety. Examples: a robot with limited sensors, or an autonomous car that cannot see around corners. In such cases, ensuring safety is harder because the agent might inadvertently take an unsafe action due to missing information. The challenge is to design SafeRL algorithms that operate under uncertainty/partial observability and still guarantee safety.

In a POMDP, the agent typically maintains a belief (distribution over states). Constraints might need to be satisfied with respect to the true state (which is unknown). For instance, we might require that \emph{for all possible true states consistent with the agent’s observations, the safety constraint holds}. This is a very strict condition and can be overly conservative. Alternatively, one might demand a high probability of safety given the belief.
There is work on \textbf{POMDPs with chance constraints}, where constraints must hold with a certain probability. Techniques often convert these into augmented state MDPs by including some memory or using scenario optimization. Another related concept is \textbf{belief shielding}: e.g., using human feedback to avoid ambiguous unsafe states. Solving safe RL in POMDPs could connect to robust control in partially observed systems (like robust Model Predictive Control with chance constraints). This problem remains largely open; progress would benefit fields like autonomous systems operating with imperfect sensors.
In many multi-agent applications, each agent has only partial, local observations (e.g., each car in traffic sees only nearby cars). A central authority that monitors and enforces safety for all agents may not exist. The challenge is to achieve safe multi-agent learning in a \emph{fully decentralized} way: each agent makes decisions based on its local view and (optionally) limited communication, and together their behaviors ensure global safety constraints are respected.

Global safety constraints often involve the joint state of multiple agents (e.g., distance between any two drones must exceed a threshold to avoid collision). No single agent can evaluate the global constraint alone. If communication is limited (bandwidth or range), agents might not know the actions or states of others in time to react safely. Moreover, learning is now on a game (or team) level, complicated by non-stationarity (each agent's environment is affected by others learning simultaneously).

Decentralized MARL has been studied (e.g., independent learners, mean-field MARL), but safety adds extra difficulty. \cite{Zhang2024Scalable} is one of few works aiming at decentralized SafeMARL. Also relevant are \textbf{distributed constrained optimization} \cite{Fioretto2018,chang2014,notarnicola2018,luang2024} in control theory where multiple controllers ensure a global constraint (like distributed frequency control in power grids \cite{Parandehgheibi2016,wang2018} ensuring safety constraints on voltage). Techniques from \textbf{graphical games} or \textbf{networked control systems} could be applied. Success in this problem would directly impact fields like distributed robotics \cite{testa2025tutorialdistributedoptimizationcooperative,wang2022distributedreinforcementlearningrobot} and network safety (e.g., ensuring no network congestion collapse via decentralized RL controllers \cite{Sunassee2021,Kiarie2025}).

SafeRL research has mostly focused on a single agent or cooperative teams. However, in the real world, multiple independent agents (e.g., companies trading stocks with safety limits, or autonomous cars from different manufacturers) may not share a common goal. They might even be adversarial. Each has safety constraints (like not going bankrupt, or not crashing) but they also have competing objectives. The challenge is to extend SafeRL to \textbf{general-sum or competitive environments}, finding appropriate equilibrium notions and algorithms to compute them.

In competitive multi-agent scenarios, one cannot simply optimize a joint objective. Methods like CPO do not directly apply, because improving one agent’s reward might hurt another’s. We need an equilibrium concept like \emph{constrained Nash equilibrium} \cite{xu2025} or \emph{constrained correlated equilibrium} \cite{Boufous2024,Chen2022}. Another issue is that safety for one agent might depend on the behavior of others. If others act recklessly, an agent might be unable to guarantee its safety without overly sacrificing reward (or might need to assume worst-case opponents).
Constrained game equilibria have been studied in economics (e.g., Nash equilibria with budget constraints). In RL, one related field is \textbf{Mean-Field Games with constraints} \cite{cannarsa2018meanfieldgamesstate,Gomez2019,Arjmand2021,Capuani2022}, but results are sparse. Another is \textbf{Multi-agent reinforcement learning for traffic} \cite{Prabuchandran2014,Zeynivand2022,li2024}, where multiple self-interested cars must avoid collisions (a safety constraint) -- some works use hand-crafted rules or potential fields, but learning such behavior while each optimizes their own objective is largely open. If this problem is solved, it could define how autonomous systems from different stakeholders safely coexist (think of air traffic control but without a central controller—planes negotiating to avoid collisions while meeting their own goals).

Consider multi-agent systems that operate over long time scales where the environment or the set of agents may change. For example, a fleet of autonomous vehicles might encounter new types of vehicles or changing traffic rules; or a robotic factory might add/remove robots over time. We need SafeMARL algorithms that can \textbf{adapt to non-stationarity} in the environment or agent population, while preserving safety. This includes scenarios like agents entering or leaving, changes in the dynamics, or even adversarial perturbations.

Non-stationarity breaks the assumptions of convergence for most RL algorithms. SafeRL adds another layer: after training, if something changes and the policy is no longer safe, the agent must detect and correct this quickly (ideally without catastrophic failure during the transition). Multi-agent adds complexity because one agent's non-stationarity (learning or adapting) is another agent's non-stationary environment.
Non-stationary RL is a growing area (sometimes framed as lifelong learning or non-stationary bandits). SafeRL in non-stationary settings has seen little study. One relevant angle is \textbf{robust safe RL}: algorithms that ensure safety under model perturbations (e.g., \cite{Chen2021SafeNonStationary,coursey2025designsafecontinualrl} considered adversarial changes in cost function within limits or continual safety in non-stationary dynamics). Another is \textbf{meta-learning safety}: a recent work by \cite{Grbic2020} attempted to meta-learn a safety critic for quickly evaluating new scenarios. Achieving continual safe learning could pave the way for real-world deployment where conditions are never static.
\subsection*{Summary of Research Directions} The problems outlined above are interconnected. For example, solving safe exploration (Problem 1) will likely benefit safe adaptation (Problem 5); and advances in decentralized safe learning (Problem 3) will be crucial for tackling competitive safe learning (Problem 4) where a central authority is absent. Each problem requires a blend of techniques—RL algorithms, optimization theory, control theory, and even insights from economics or game theory.Crucially, addressing these problems will move SafeRL from a laboratory curiosity to a dependable component of autonomous systems. We expect that success in these areas will result in publishable work at top venues (NeurIPS, ICML, ICRA, etc.), given the importance and difficulty of ensuring safety in learning systems. By formulating them here, we hope to encourage more researchers to contribute to these challenges.

\section{Conclusion} Safe Reinforcement Learning is a vital area of research for deploying learning agents in real-world environments where failures are costly or dangerous. In this survey, we provided a detailed overview of SafeRL with a focus on the CMDP framework for incorporating constraints. We reviewed theoretical foundations including CMDP definitions, Lagrangian duality, and solution methods like linear programming and policy gradient for constrained problems. Building on this foundation, we discussed state-of-the-art SafeRL algorithms such as Constrained Policy Optimization, Lagrange multiplier methods, safe exploration via shielding, and how these have been extended to multi-agent settings.Our survey highlights that: \begin{itemize} 
\item SafeRL is inherently a cross-disciplinary field, drawing from machine learning, optimal control, and formal methods. The CMDP formulation provides a unifying language for many approaches. 
\item There is a rich toolbox of algorithms for single-agent SafeRL that can achieve good performance while respecting constraints, though each has trade-offs in terms of safety guarantees vs. efficiency. 
\item SafeMARL is a frontier with significant potential impact (e.g., fleet management, multi-robot systems). Early algorithms like MACPO and shielding strategies show feasibility, but general solutions for decentralized and competitive scenarios are still lacking. 
\end{itemize}
We identified several open research problems that require further work: from guaranteeing zero violations to handling partial observability, and from fully decentralized safe coordination to safe learning in non-stationary multi-agent environments. These problems underscore that SafeRL is not a solved problem—there are theoretical challenges (ensuring safety and convergence), practical issues (scalability, function approximation errors), and new frontiers (multi-agent interactions). In conclusion, SafeRL offers a pathway to more trustworthy AI systems by marrying reinforcement learning with constraint satisfaction. As RL agents become more capable and autonomous, ensuring their safety will be paramount. We hope this survey serves as a useful resource for researchers to understand the current landscape and to inspire further advances. The continued development of SafeRL methods will help unlock applications of RL in domains that are currently out of reach due to safety concerns, ultimately enabling AI to make beneficial decisions without posing undue risk.

\section*{LLM Usage}
ChatGPT-5.1 free version was used for polishing of initial draft text to improve the English and grammar. Claude Code was used to assist with TikZ figure creation, formatting, and bibliography management.

\bibliographystyle{plainnat}

\bibliography{references}

\end{document}